\pdfoutput=1

\PassOptionsToPackage{most}{tcolorbox}
\newcommand{\tamulogofooter}{}
\documentclass[11pt,letterpaper]{mystyle}
\usepackage{fancyhdr}

\usepackage{tikz}
\usepackage[utf8]{inputenc} 
\usepackage[T1]{fontenc}
\usepackage[numbers]{natbib}
\usepackage{adjustbox}
\usepackage{breakurl}    
\usepackage{hyperref}
\usepackage[edges]{forest}
\usepackage{subcaption}
\usepackage{soul}
\usepackage{multirow}
\usepackage[utf8]{inputenc} 
\usepackage{booktabs}       
\usepackage{amsfonts}       
\usepackage{nicefrac}      
\usepackage{microtype}     
\usepackage{xcolor}        
\usepackage{colortbl}
\usepackage{enumitem}
\usepackage{amssymb}
\usepackage{wrapfig}
\usepackage{courier}
\usepackage{bxcoloremoji}
\usepackage{CJKutf8}
\usepackage{ragged2e}
\usepackage{longtable}
\usepackage[useregional]{datetime2}
\usepackage{threeparttable}
\usepackage{cleveref} 
\usepackage{wrapfig}
\usepackage{amsmath}
\usepackage{makecell}
\usepackage{tcolorbox}
\tcbuselibrary{theorems}

\newtcbtheorem{proposition}{Proposition}{
    colback=gray!5,
    colframe=black!75,
    fonttitle=\bfseries
}{prop}

\crefname{appsec}{Appendix}{Appendices}
\Crefname{appsec}{Appendix}{Appendices}

\definecolor{tamuMaroon}{HTML}{500000}
\colorlet{tamuGrayMaroon}{tamuMaroon!15!black}
\newcommand{\github}{\raisebox{-1.5pt}{\includegraphics[height=1.05em]{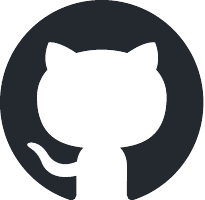}}}
\renewcommand{\today}{\DTMtoday}

\fancypagestyle{headstyle}{
    \fancyhead[C]{
        Spin-Weighted Spherical Harmonics Enable Complete and Scalable \protect$\mathrm{E}(3)$-Equivariant Networks
    }
}

\definecolor{hidden-red}{RGB}{205, 44, 36}
\definecolor{hidden-blue}{RGB}{194,232,247}
\definecolor{hidden-orange}{RGB}{243,202,120}
\definecolor{hidden-green}{RGB}{34,139,34}
\definecolor{hidden-pink}{RGB}{255,245,247}
\definecolor{hidden-black}{RGB}{20,68,106}
\definecolor{purple}{RGB}{144,153,196}
\definecolor{yellow}{RGB}{255,228,123}
\definecolor{hidden-yellow}{RGB}{255,248,203}
\definecolor{tkcolor}{RGB}{224,223,255}
\definecolor{darkblue}{rgb}{0, 0.40, 0.75}

\definecolor{tamuMaroon}{HTML}{500000}
\colorlet{abstractTextColor}{tamuMaroon!15!black}
\newcommand{\abstractstyle}{\color{abstractTextColor}}

\makeatletter
\let\oldabstract\abstract
\let\oldendabstract\endabstract
\renewcommand{\abstract}{\oldabstract\abstractstyle}
\renewcommand{\endabstract}{\oldendabstract}
\makeatother

\hypersetup{colorlinks=true, citecolor=darkblue, linkcolor=darkblue, urlcolor=darkblue}
\tcbset{
  aibox/.style={
    width=\linewidth,
    top=8pt,
    bottom=4pt,
    colback=blue!6!white,
    colframe=black,
    colbacktitle=black,
    enhanced,
    center,
    attach boxed title to top left={yshift=-0.1in,xshift=0.15in},
    boxed title style={boxrule=0pt,colframe=white,},
  }
}

\newtcolorbox{TakeawayBox}[2][]{takeawaybox,title=#2,#1}

\title{Spin-Weighted Spherical Harmonics Enable Complete and Scalable E(3)-Equivariant Networks}

\author{
  Chenxing Liang$^{1*}$, 
  Yuchao Lin$^{1*}$, 
  Andrii Kryvenko$^{1*}$, 
  Wendi Yu$^{1}$, 
  \textbf{Chuan Li}$^{2}$, \\
  \textbf{Jianwen Xie}$^{2}$, 
  \textbf{Xiaofeng Qian}$^{3,4,5}$, 
  \textbf{Shuiwang Ji}$^{1,3,6}$
\\

\vspace{1mm}

\normalfont{
$^1$ Department of Computer Science and Engineering, Texas A\&M University\vspace{-5pt} \\
$^2$ Lambda, Inc.\vspace{-5pt} \\
$^3$ Department of Materials Science and Engineering, Texas A\&M University\vspace{-5pt} \\
$^4$ Department of Electrical and Computer Engineering, Texas A\&M University\vspace{-5pt} \\
$^5$ Department of Physics and Astronomy, Texas A\&M University\vspace{-5pt} \\
$^6$ Department of Mechanical Engineering, Texas A\&M University\vspace{-5pt} \\
}}

\begin{document}

\begin{abstract}

  \textbf{\large Abstract:}
  \vspace{1mm}

$\mathrm{E}(3)$-equivariant networks are promising for 3D atomistic system modeling, yet their scalability is limited by the $O(L^6)$ complexity of the Clebsch-Gordan Tensor Product (CGTP). The recently proposed Gaunt Tensor Product (GTP) reduces the complexity but is unable to capture the antisymmetric paths, resulting in incomplete expressivity. In this work, we present \textit{SpinGTP}, an approach to overcome the GTP incompleteness by generalizing from scalar functions to Spin-Weighted Spherical Harmonics (SWSH). By relying on the algebraic properties of SWSH, SpinGTP recovers the missing antisymmetric interactions while maintaining the asymptotic efficiency of GTP. It also allows for a more expressive equivariant basis that naturally accounts for the parity-odd components of tensor products. We evaluate SpinGTP across diverse benchmarks, including Tetris, 3BPA, SPICE-MACE-OFF, and OC20. Our results show that SpinGTP achieves accuracies comparable to full CGTP. Notably, by explicitly capturing antisymmetric paths, SpinGTP exhibits superior performance in tasks involving chiral materials and non-centrosymmetric geometries. This work provides a complete, scalable, and mathematically rigorous path toward high-order equivariance in large-scale 3D atomistic system simulations. Our code is publicly available at \url{https://github.com/divelab/SpinGTP/}.

  \vspace{5mm}

  

  \vspace{1mm}
  \textbf{Keywords}: AI for Science, Equivariant
Networks, Spin-Weighted Spherical Harmonics, 3D Atomistic Simulations
  \vspace{6mm}








   \textbf {* These authors contributed equally}

  \vspace{3mm}


  \github{} \textbf{Github Repository}: 
\url{https://github.com/divelab/SpinGTP/}



    \vspace*{0.2in}

\end{abstract}

\maketitle

\pagestyle{headstyle}
\thispagestyle{empty}

\newpage
\vspace{2em}
\tableofcontents

\newpage

\section{Introduction}
\label{sec:intro}
The integration of physical symmetries into deep learning models has become a cornerstone of modern artificial intelligence for physical sciences~\cite{zhang2025artificial,bronstein2021geometric,villar2021scalars,kondor2025principles,fei2024rotation,fu2025augmenting}. By enforcing equivariance, neural networks can model 3D atomic environments with superior data efficiency and generalization compared to standard architectures. Central to these models is the Clebsch-Gordan Tensor Product (CGTP), a fundamental operation that enables interactions between features of different angular frequencies, known as irreducible representations (irreps)~\cite{khersonskii1988quantum, thomas2018tensor, anderson2019cormorant}. Despite its expressive power, the CGTP suffers from a computational complexity of $O(L^6)$ with the maximum angular degree $L$. This complexity often forces practitioners to limit $L$ to small values, sacrificing the high-order geometric information necessary for modeling complex molecular interactions and interatomic potentials.

Recent efforts to overcome this bottleneck have focused on accelerating tensor products through alternative mathematical formulations~\citep{passaro2023reducing,lin2025tensor}. A notable advance is the Gaunt Tensor Product (GTP) \cite{luo2024enabling}, which maps tensor products of irreps to pointwise multiplications of spherical functions in 2D Fourier basis. By leveraging the convolution theorem and Fast Fourier Transforms (FFT), GTP reduces the asymptotic complexity from $O(L^6)$ to $O(L^3)$. 
However, this efficiency comes with an expressivity loss. GTP retains only coupling paths where the sum of irrep degrees $\ell_1+\ell_2+\ell_3$ is even, and removes odd degree-sum paths~\citep{xie2024price}. These missing paths include antisymmetric interactions such as the $[\ell_1=1,\ell_2=1,\ell_3=1]$ vector cross-product path, as well as couplings needed to form pseudoscalar and axial quantities. This limits scalar GTP on parity-sensitive tasks, especially those involving chiral geometries. We refer to this limitation as the antisymmetric gap.

To address the antisymmetric gap, recent work~\cite{xie2026asymptotically} shows that the antisymmetric paths in scalar GTP can be recovered by lifting scalar spherical signals to irrep-valued signals. Using Vector Spherical Harmonics (VSH), they introduce the Vector Signal Tensor Product (VSTP), which recovers missing antisymmetric paths through generalized Gaunt products with Wigner $9j$ symbol. 



In this work, we pursue a more direct mathematical path to this problem by introducing the Spin-Weighted Gaunt Tensor Product (SpinGTP). Rather than lifting scalar spherical signals to vector-valued signals, SpinGTP works directly with Spin-Weighted Spherical Harmonics (SWSHs) and their generalized Gaunt integral with Wigner $3j$ symbols. The spin-weight index replaces the scalar zero-order coupling with a signed spin selection rule, allowing odd degree-sum paths that scalar GTP removes when the required spin sectors are present. We implement this operator in a real, parity-labeled SWSH basis for a rich and concrete representation.

We evaluate SpinGTP on a suite of benchmarks, including the Tetris and large-scale atomistic datasets such as 3BPA \cite{3BPA}, SPICE-MACE-OFF \cite{kovacs2025mace}, and OC20 \cite{chanussot2021open}. Our results demonstrate that SpinGTP-based networks achieves accuracy comparable to full $O(L^6)$ CGTP models while remaining as computationally efficient as the original GTP. Furthermore, we show that our model outperforms previous tensor product methods in predicting properties of chiral materials, showing that the inclusion of antisymmetric paths through SWSH is beneficial for 3D geometric deep learning of chiral geometries.

Our contributions are summarized as follows:
\begin{itemize}[leftmargin=15pt]
    \item We propose SpinGTP, an equivariant operation based on SWSH that restores the mathematical completeness to the Gaunt Tensor Product framework.
    \item We provide mathematical derivations showing that the SpinGTP formulation efficiently recovers the previously missing parity-odd interactions (antisymmetric paths) required for universal $\mathrm{E}(3)$ equivariance.
    \item We evaluate SpinGTP across a range of tasks, including Tetris, 3BPA,
SPICE-MACE-OFF, and OC20, demonstrating accuracy comparable to full CGTP
while yielding targeted improvements in chiral geometry modeling through the
recovery of antisymmetric paths.
\end{itemize}

\section{Related Work}

\textbf{Equivariant Network Architectures.} 
The landscape of $\mathrm{E}(3)$-equivariant modeling has evolved from early invariant descriptors such as SchNet \cite{SchNet} and DimeNet \cite{DimeNet++} toward steerable frameworks that leverage irreducible representations (irreps). Seminal architectures like NequIP \cite{nequip} and Allegro \cite{musaelian2023learning} demonstrated the superior sample efficiency of Clebsch-Gordan (CG) tensor products, while the Equiformer series~\cite{equiformer,equiformerv2} extended these principles to attention-based mechanisms. To alleviate the computational burden of edge-wise products, e2former\cite{li2026eformer} introduced a node-wise message-passing scheme. While alternative scalarization methods, such as PaiNN \cite{schutt2021equivariant} and NewtonNet \cite{haghighatlari2022newtonnet} offer increased throughput, they often bypass the full tensor product space, potentially sacrificing the formal universality required for complex geometric.

\textbf{Tensor Product Acceleration.}
The Clebsch-Gordan Tensor Product (CGTP) is a standard interaction operator in $\mathrm{E}(3)$-equivariant neural networks~\cite{khersonskii1988quantum}. However, direct CGTP scales as $O(L^6)$ with the maximum angular degree $L$, which limits the use of high-order features. Several recent methods reduce this cost by exploiting structure in equivariant architectures. eSCN~\cite{passaro2023reducing} aligns features with the edge direction during message passing, which sparsifies the equivariant convolution and reduces it to an $\mathrm{SO}(2)$ computation in the local edge frame. ~\cite{lin2025tensor} reduces tensor-product cost through low-rank tensor decomposition structure.~\cite{luo2024enabling} proposed the Gaunt Tensor Product (GTP), which evaluates interactions through spherical harmonic transforms and pointwise products, reducing the complexity to $O(L^3)$. GTP is efficient, but scalar GTP removes odd $\ell$-sum paths, including antisymmetric interactions such as the vector cross product~\cite{xie2024price}. To restore expressivity without reverting to $O(L^6)$ scaling, a Vector Spherical Harmonic (VSH) basis and $9j$ recoupling are introduced to recover missing antisymmetric paths \cite{xie2026asymptotically}. This framework provides complete algorithm to achieve true asymptotic speedups, reaching $O(L^4 \log^2 L)$ complexity via fast spectral transforms.




\begin {comment}
\section{Preliminaries}

\yuchao{We can maintain one to two equations for each, otherwise it is overloaded.}

\textbf{Gaunt Tensor Product (GTP).} 
To address the $O(L^6)$ complexity of the Clebsch-Gordan Tensor Product (CGTP), Luo et al.~\citep{luo2024enabling} proposed the Gaunt Tensor Product, which reinterprets the interaction of irreducible representations (irreps) as a multiplication of functions in the spatial domain.

The framework is built upon the \textit{Gaunt coefficient}, defined as the integral of the product of three scalar spherical harmonics over the unit sphere $\mathrm{S}^2$:
\begin{equation}
\label{eq:gaunt_coefficient}
\mathcal{G}(\ell_1, m_1, \ell_2, m_2, \ell_3, m_3) = \int_{\mathrm{S}^2} Y_{\ell_1 m_1}(\Omega) Y_{\ell_2 m_2}(\Omega) Y_{\ell_3 m_3}^*(\Omega) \, d\Omega.
\end{equation}
This integral represents the coupling strength between three angular momentum states. The key innovation of GTP is using the mathematical identity that connects these integrals to the Wigner 3-j symbols, and consequently to the Clebsch-Gordan coefficients:
\begin{equation}
\label{eq:gaunt_cg_bridge}
\mathcal{G}(\ell_1, m_1, \ell_2, m_2, \ell_3, m_3) = \sqrt{\frac{(2\ell_1+1)(2\ell_2+1)}{4\pi(2\ell_3+1)}} C_{(\ell_1, m_1)(\ell_2, m_2)}^{(\ell_3, m_3)} C_{(\ell_1, 0)(\ell_2, 0)}^{(\ell_3, 0)}.
\end{equation}
Eq.~\ref{eq:gaunt_cg_bridge} implies that the interaction between two equivariant features $\hat{\mathbf{a}}_{\ell_1 m_1}$ and $\hat{\mathbf{b}}_{\ell_2 m_2}$ can be computed by first projecting them into the spatial domain as spherical functions, $a(\Omega) = \sum \hat{a}_{\ell m} Y_{\ell m}(\Omega)$ and $b(\Omega) = \sum \hat{b}_{\ell m} Y_{\ell m}(\Omega)$. In this domain, the tensor product interaction is simply the pointwise multiplication. To achieve $O(L^3)$ efficiency, the signals are sampled on an equiangular grid and transformed into a 2D Fourier basis $\mathcal{F}$. By the convolution theorem, the multiplication is computed efficiently using Fast Fourier Transforms (FFT):
\begin{equation}
\label{eq:gaunt_fft_op}
\hat{\mathbf{c}}_{u,v} = \text{FFT} \left( \text{IFFT}(\hat{\mathbf{a}}_{u,v}) \odot \text{IFFT}(\hat{\mathbf{b}}_{u,v}) \right),
\end{equation}
where $\odot$ denotes pointwise multiplication and $(u,v)$ are the Fourier indices. We provide further mathmatical details regarding the grid sampling and the specific Fourier-basis conversion in Appendix \ref{app:gaunt_details}. 

\textbf{Spin-weighted Spherical Harmonics (SWSH).}
 Originally developed by Goldberg et al.~\citep{goldberg1967spin} for the study of gravitational radiation, SWSHs are functions on the sphere characterized by an integral spin weight $s \in \mathbb{Z}$ in addition to the standard degree $\ell$ and order $m$, with the constraint $|s| \leq \ell$. 

Geometrically, while standard harmonics ($s=0$) are scalar fields, SWSHs for $s \neq 0$ are sections of spin-weighted line bundles over $\mathrm{S}^2$. Under a local rotation of the tangent plane by an angle $\chi$, these functions transform as ${}_sY_{\ell m} \to {}_sY_{\ell m} e^{is\chi}$. They form a complete orthonormal basis for functions of a fixed spin weight:
\begin{equation}
    \int_{S^{2}} {}_sY_{\ell m} \, {}_sY_{\ell'm'}^* \, d\Omega = \delta_{\ell \ell'} \delta_{mm'}.
\end{equation}
The significance of SWSHs for equivariant learning lies in their triple integral properties. For three harmonics with spin weights satisfying the selection rule $s_1 + s_2 + s_3 = 0$, the \textit{Generalized Gaunt Integral} is:
\begin{equation}
\label{eq:generalized_gaunt}
\begin{aligned}
    \int_{\mathrm{S}^2} {}_{s_{1}}Y_{\ell_{1}m_{1}} \, {}_{s_{2}}Y_{\ell_{2}m_{2}} \, {}_{s_{3}}Y_{\ell_{3}m_{3}}^{*} \, d\Omega = & \sqrt{\frac{(2\ell_1+1)(2\ell_2+1)(2\ell_3+1)}{4\pi}} \\
    & \times \begin{pmatrix} \ell_{1} & \ell_{2} & \ell_{3} \\ m_{1} & m_{2} & -m_{3} \end{pmatrix} 
    \begin{pmatrix} \ell_{1} & \ell_{2} & \ell_{3} \\ s_{1} & s_{2} & -s_{3} \end{pmatrix}.
\end{aligned}
\end{equation}
Comparing Eq.~\ref{eq:generalized_gaunt} with the scalar identity in Eq.~\ref{eq:gaunt_cg_bridge}, we observe that the scalar ``zero-frequency'' term $C_{(\ell_1, 0)(\ell_2, 0)}^{(\ell_3, 0)}$ is replaced by a Wigner 3-j symbol involving the spin weights $s_i$. Crucially, whereas the scalar term vanishes for $\ell_1 + \ell_2 + \ell_3 \in 2\mathbb{Z} + 1$, the spin-weighted 3-j symbol is non-zero for \textit{both} even and odd parity sums when $s_i \neq 0$. This mathematical shift allows the \textbf{Spin-Weighted Gaunt Tensor Product} to recover the full spectrum of Clebsch-Gordan interactions—including antisymmetric paths—while maintaining the $O(L^3)$ efficiency of the FFT-based multiplication. We provide the specific differential operators used to generate SWSHs and the proof of completeness in Appendix \ref{app:swsh_theory}.

\end{comment}

\section{Gaunt Tensor Product of Spin-Weighted Spherical Harmonics} 


This section presents the methodology for the Spin-Weighted Gaunt Tensor Product (SpinGTP). We first review the standard Gaunt Tensor Product and identify its main expressivity limitation, the loss of antisymmetric paths with odd $\ell$-sum. We then introduce Spin-Weighted Spherical Harmonics (SWSH) and their spin-weighted Gaunt integral. To use this in a real-valued neural network, we build real SWSH bases and extend them with parity labels, allowing the model to represent antisymmetric paths when the required spin channels are present. Finally, we summarize the implementation strategy and the specialized equivariant layers in the architecture.

\subsection{Gaunt Tensor Product Efficiency and the Antisymmetry Gap} 
\label{sec:gtp_gap}


To reduce the $\mathcal{O}(L^6)$ cost of standard Clebsch-Gordan contractions, the scalar Gaunt Tensor Product (GTP)~\citep{luo2024enabling} represents equivariant interactions as pointwise products of spherical signals. The resulting coupling is given by a Gaunt coefficient, defined as the integral of two input spherical harmonics against an output spherical harmonic over $\mathrm{S}^2$ such that
{\small
\begin{equation}
\label{eq:gaunt_bridge}
G^{(\ell_3,m_3) }_{(\ell_1, m_1)(\ell_2, m_2)} = \int_{\mathrm{S}^2} Y_{\ell_1 m_1} Y_{\ell_2 m_2} Y_{\ell_3 m_3}^* d\Omega = \sqrt{\frac{(2\ell_1+1)(2\ell_2+1)}{4\pi(2\ell_3+1)}} C_{(\ell_1, m_1)(\ell_2, m_2)}^{(\ell_3, m_3)} C_{(\ell_1, 0)(\ell_2, 0)}^{(\ell_3, 0)}.
\end{equation}
}


As shown in~\Cref{eq:gaunt_bridge}, GTP retains a subset of Clebsch-Gordan paths where the scalar coupling $C_{(\ell_1, 0)(\ell_2, 0)}^{(\ell_3, 0)}$ is nonzero. This restriction simplifies the interaction, and the spatial product formulation enables an $\mathcal{O}(L^3)$ implementation via fast Fourier transform. The resulting computation can scale to high spectral resolutions that are costly for direct Clebsch-Gordan contractions.


Despite this efficiency, scalar GTP is constrained by the factor $C_{(\ell_1, 0)(\ell_2, 0)}^{(\ell_3, 0)}$, which vanishes whenever $\ell_1+\ell_2+\ell_3$ is odd. This removes all odd $\ell$-sum coupling paths, including antisymmetric interactions such as the $[\ell_1 = 1,\ell_2 = 1,\ell_3 = 1]$ vector cross-product path. As a result, scalar GTP can fail to distinguish configurations whose signal depends on pseudoscalar or chiral interactions. This limitation motivates the spin-weighted extension introduced below.

\subsection{Spin-weighted Spherical Harmonics for Restoring Expressivity}
\label{sec:swsh_gaunt}


Spin-Weighted Spherical Harmonics (SWSHs)~\citep{goldberg1967spin} generalize ordinary spherical harmonics by adding an integer spin weight $s$, alongside the degree $\ell$ and order $m$, with $|s| \leq \ell$. The case $s=0$ recovers ordinary spherical harmonics. For $s \neq 0$, SWSHs represent spin-weighted fields on $\mathrm{S}^2$ and transform by a phase under rotations of the local tangent frame. Additional details are provided in~\Cref{sec:swsh}.


The role of SWSHs in our tensor product comes from their generalized Gaunt integral. For three SWSHs with the spin weights satisfying $s_3=s_1+s_2$, the Gaunt coefficient is
\begin{equation}
\label{eq:generalized_gaunt}
\begin{aligned}
G^{(\ell_3,m_3,s_3)}_{(\ell_1,m_1,s_1)(\ell_2,m_2,s_2)} &=\int_{\mathrm{S}^2} {}_{s_{1}}Y_{\ell_{1}m_{1}} \, {}_{s_{2}}Y_{\ell_{2}m_{2}} \, {}_{s_{3}}Y_{\ell_{3}m_{3}}^{*} \, d\Omega \\
&= \sqrt{\frac{\prod_{i=1}^3 (2\ell_i+1)}{4\pi}} \begin{pmatrix} \ell_{1} & \ell_{2} & \ell_{3} \\ m_{1} & m_{2} & -m_{3} \end{pmatrix} \begin{pmatrix} \ell_{1} & \ell_{2} & \ell_{3} \\ s_{1} & s_{2} & -s_{3} \end{pmatrix}.
\end{aligned}
\end{equation}
Unlike the scalar factor in~\Cref{eq:gaunt_bridge}, the spin-weighted factor can be nonzero for odd $\ell_1+\ell_2+\ell_3$ when proper nonzero spin weights are present. SpinGTP uses this degree of freedom to recover antisymmetric paths that are absent from scalar GTP, while retaining the Gaunt product structure. A more formal statement is presented below and the proof is provided in~\cref{sec:swsh}.

\begin{proposition}{Spin-weighted path completion}{}
Let $(\ell_1,\ell_2,\ell_3)$ satisfy the Clebsch-Gordan triangle rule. Then there exist signed spin weights
\[
s_1\in[-\ell_1,\ell_1],\qquad
s_2\in[-\ell_2,\ell_2],\qquad
s_3\in[-\ell_3,\ell_3],
\]
with $s_3=s_1+s_2$, such that the signed-spin path $(\ell_1,s_1)\otimes(\ell_2,s_2)\to(\ell_3,s_3)$
is not identically zero.
\end{proposition}

\textbf{Local Frames and Gauge Dependency.} Spin-weighted features depend on a local tangent-frame gauge. Unlike scalar spherical harmonics, an SWSH value ${}_sY_{\ell m}(\mathbf{x})$ is defined relative to a local orthonormal frame in the tangent plane at $\mathbf{x}$. If this frame is rotated by an angle $\psi$, a spin-$s$ field transforms as
\begin{equation}
    f_s(\mathbf{x}) \mapsto e^{is\psi} f_s(\mathbf{x}).
\end{equation}
This gauge dependence means that the basis functions implicitly capture the orientation of the spherical surface, yielding a richer representation than scalar spherical harmonics. However, spin-weighted features can be combined consistently only when their local frame conventions are compatible. In our implementation, frame construction is treated as part of the equivariant architecture. We use explicit geometric frames to evaluate SWSH edge features, with the details described in~\Cref{sec:equiv}.

\textbf{Best Asymptotic Runtime Complexity.} For bounded spin weight \(|s|_{\max}=1\), SpinGTP can attain the same \(O(L^4\log^2 L)\) asymptotic complexity as the VSTP~\cite{xie2026asymptotically}, provided that the FFT-based Gaunt tensor product~\cite{luo2024enabling} is used. A proof is given in~\Cref{sec:time_cost}.

\begin{figure}[t]
    \centering
    \includegraphics[width=0.8\textwidth, trim={2cm 4cm 2cm 3.5cm}, clip]{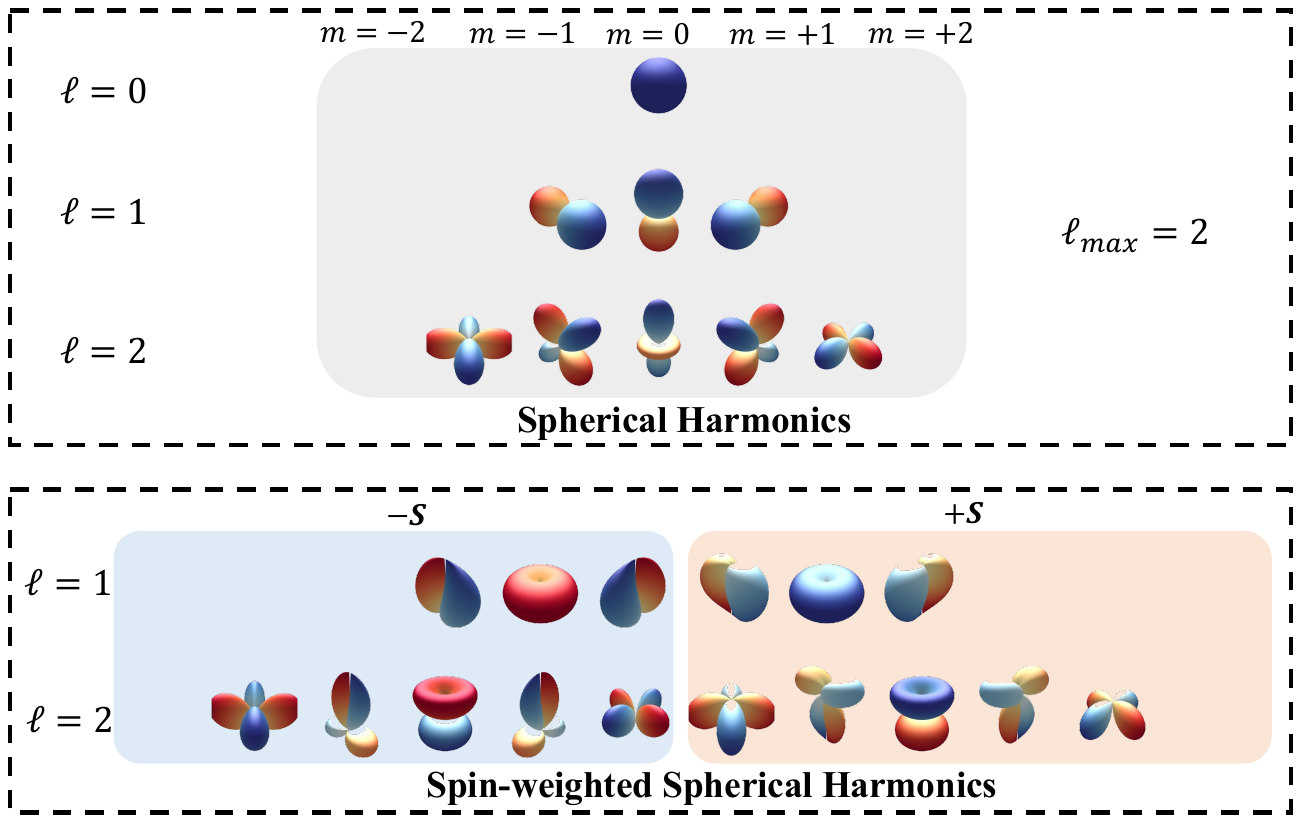}
    \caption{{Comparison between standard and spin-weighted spherical harmonics.} 
    (a) Scalar spherical harmonics with $\ell_{max}=2$. The rows represent the basis functions for $\ell=0, 1,$ and $2$, respectively, with the magnetic index $m$ ranging from $-\ell$ to $\ell$. 
    (b) Spin-weighted spherical harmonics with $\ell_{max}=2$ and absolute spin $|s|=1$. The rows display the basis functions for degrees $\ell=1$ and $\ell=2$. For each degree, the basis includes components for both $-s$ and $+s$ across full range of $m$.  
    }
    \label{fig:Figure 1}
\end{figure}

\subsection{Real-Basis Spin-Weighted Spherical Harmonics}
\label{sec:real_swsh_main}

SWSHs are naturally complex-valued, while our network operates on real-valued feature tensors to maintain compatibility with standard activation functions and optimize memory efficiency. For scalar harmonics, the usual real basis can be constructed within a fixed degree $\ell$. For spin-weighted harmonics, the construction must also account for the relation
\begin{equation}
\label{eq:swsh_conjugation}
{}_sY_{\ell m}^* = (-1)^{m+s} {}_{-s}Y_{\ell, -m}.
\end{equation}

Following~\citep{reisswig2013general}, we satisfy~\Cref{eq:swsh_conjugation} by using a \textbf{real SWSH basis} that couples the positive and negative spin weights as
\begin{equation}
\label{eq:real_swsh}
{}_sR_{\ell m} =
\begin{cases}
\frac{i}{\sqrt{2}}\left( {}_sY_{\ell, -m} - (-1)^{-m+s} {}_{-s}Y_{\ell m} \right) & \text{if } m < 0 \\
{}_sY_{\ell 0} & \text{if } m = 0 \\
\frac{1}{\sqrt{2}}\left( {}_sY_{\ell m} + (-1)^{m+s} {}_{-s}Y_{\ell, -m} \right) & \text{if } m > 0.
\end{cases}
\end{equation}
For any $|s|>0$, concatenating both ${}_sY_{\ell m}$ and ${}_{-s}Y_{\ell m}$ increases the basis dimension from $2\ell+1$ to $2(2\ell+1)$ because both signed spin sectors must be retained for real SWSH construction. This doubled representation provides the real spin-weighted features used by the network. A visualized comparison between the real scalar spherical harmonics and the real basis SWSH is provided in~\Cref{fig:Figure 1}.


\paragraph{Transformation via the $\mathbf{Q}$ Matrix.}

To connect real feature vectors $\mathbf{a}_{\text{real}}$ with the complex Gaunt coefficients used in the tensor product, we use a unitary change of basis $\mathbf{Q}$ such that
\[
\mathbf{a}_{\text{complex}}=\mathbf{Q}\mathbf{a}_{\text{real}} .
\]
When $|s|=0$, this is the standard real-to-complex spherical harmonic transformation. For a fixed $\ell$ and $|s|>0$, the matrix satisfies $\mathbf{Q}\in\mathbb{C}^{2(2\ell+1)\times 2(2\ell+1)}$ and is organized into four blocks that couple the $\pm s$ and $\pm m$ indices. The explicit block-wise definitions of $\mathbf{Q}$ and the corresponding real-form Gaunt coefficients are given in~\Cref{sec:real_swsh}.

\subsection{Parity-Equivariant Spin-Weighted Spherical Harmonics}
\label{sec:parity_swsh_main}


An $\mathrm{E}(3)$-equivariant model must distinguish signals with different parities. Standard scalar spherical harmonics have a fixed parity $P=(-1)^\ell$. For spin-weighted harmonics, spatial inversion also flips the spin weight such that
\[
{}_sY_{\ell m}(-\mathbf{x}) = (-1)^\ell {}_{-s}Y_{\ell m}(\mathbf{x}).
\]
Thus inversion maps the $s$ sector to the $-s$ sector. As a result, the real basis functions ${}_sR_{\ell m}$ defined in~\Cref{eq:real_swsh} do not by themselves carry a fixed parity when $s\neq 0$. 
To obtain basis functions with an inversion parity $p \in \{+1,-1\}$, we combine the paired spin sectors and define a \textbf{parity-equivariant real SWSH basis}, denoted by $\mathcal{R}_{p}^{s,\ell,m}(\mathbf{x})$. This construction separates the parity-even and parity-odd components as
\begin{equation}
\label{eq:parity_basis}
\mathcal{R}_{p}^{s,\ell,m}(\mathbf{x}) = \frac{1}{\sqrt{2}} \left( {}_{-s}R_{\ell m}(\mathbf{x}) + p(-1)^\ell \, {}_{s}R_{\ell m}(\mathbf{x}) \right).
\end{equation}
By construction, this basis satisfies the required parity constraint $\mathcal{R}_{p}^{s,\ell,m}(-\mathbf{x}) = p \mathcal{R}_{p}^{s,\ell,m}(\mathbf{x})$. 
This allows the representation to distinguish scalar from pseudoscalar channels and polar vector from axial vector channels. A detailed derivation is provided in~\Cref{sec:parity_derivation}.


The parity-labeled real basis has two roles. First, it lets the network track the inversion parity of each irrep throughout tensor product layers. Second, it provides the necessary geometric signals for physical properties whose sign changes under reflection, including pseudoscalar quantities associated with chiral geometry.



\subsection{High-Performance Implementation of SpinGTP}
\label{sec:swsh_tp_impl}


The computational core of SpinGTP is an instruction-based tensor product over SWSH irreps. Each irrep is indexed by $(\ell, |s|, p)$. Given two input feature tensors $\mathbf{x}^{(1)}$ and $\mathbf{x}^{(2)}$, the weighted tensor product maps pairs of input irreps to admissible output irreps as
{
\small
\begin{equation}
\label{eq:spin_gtp}
[\mathbf{x}^{(1)} \otimes_{\mathbf{W}} \mathbf{x}^{(2)}]_{\ell_3,s_3,p_3}^{(m_3, w)} = \sum_{u, v} \mathbf{W}_{uvw} \sum_{m_1, m_2} G_{(\ell_1, m_1, s_1), (\ell_2, m_2, s_2)}^{(\ell_3, m_3, s_3)} \mathbf x_{u, \ell_1, m_1, p_1}^{(1)} \mathbf x_{v, \ell_2, m_2, p_2}^{(2)},\quad p_3=p_1p_2,
\end{equation}
}
where $G$ denotes the real SWSH Gaunt contraction and $\mathbf{W}$ mixes feature multiplicities. Internally, the contraction enumerates the signed spin sectors $s=\pm |s|$, applies the selection rule $s_3=s_1+s_2$, combines the paired output spin sectors according to the output parity. To balance expressivity with parameter efficiency, we implement \textit{fully connected}, \textit{depthwise}, and \textit{element-wise} tensor products. The implementation details are shown in~\Cref{sec:tp_impl}.

\textbf{Pre-contracted Kernel Execution.}
A direct evaluation of SpinGTP involves many small sparse contractions, which is inefficient on GPUs. Let B denote the number of edges, and let U and V denote the input multiplicities. When computing over all $m_3$, this contraction has time complexity $O(BUV\ell_1 \ell_2\ell_3)$. In molecular message passing, the second input $\mathbf{x}^{(2)}$ is often an edge feature with multiplicity $V=1$. We exploit this structure by precontracting the tensor product kernel with $\mathbf{x}^{(2)}$ before applying it to the node channels.

Based on~\Cref{eq:spin_gtp}, when the edge multiplicity is one, the sum over $v$ disappears and we first form an edge-dependent kernel
{
\small
\begin{equation}
\mathbf{K}_{\mathrm{pre}}^{(\ell_3,s_3,p_3)}(m_1,m_3)
=
\sum_{m_2}
G_{(\ell_1,m_1,s_1),(\ell_2,m_2,s_2)}^{(\ell_3,m_3,s_3)}
\mathbf x_{\ell_2,m_2,p_2}^{(2)} .
\end{equation}
}
The remaining contraction becomes
{
\small
\begin{equation}
[\mathbf{x}^{(1)} \otimes_{\mathbf{W}} \mathbf{x}^{(2)}]_{\ell_3,s_3,p_3}^{(m_3, w)}
=
\sum_u
\mathbf{W}_{u1w}
\sum_{m_1}
\mathbf{K}_{\mathrm{pre}}^{(\ell_3,s_3,p_3)}(m_1,m_3)
\mathbf x_{u,\ell_1,m_1,p_1}^{(1)} .
\end{equation}
}
This reduces repeated geometric contractions across node channels, and gives a time complexity $O(B\ell_1\ell_2\ell_3 + BU\ell_1\ell_3)$ when computing over all $m_3$. In implementation, the precontracted kernel is evaluated as a dense batched matrix multiplication, followed by the learned multiplicity mixing and scatter into the output irreps. The time ablation of this implementation is provided in~\Cref{sec:runtime_comparison}.

\subsection{Specialized SWSH Equivariant Layers}
\label{sec:swsh_layers}

The SWSH representation requires standard neural network layers to respect the irrep structure indexed by $(\ell, |s|, p)$. We use SWSH linear and normalization layers that mix only compatible feature channels and preserve the type of each irrep.

\textbf{Equivariant Linear Mixing.} The linear layer performs channel mixing within each SWSH irrep. Since features are grouped by $(\ell, |s|, p)$, the linear map is block diagonal across irrep types such that
\begin{equation}
    [\mathrm{Linear}(\mathbf{x})]_{\ell,s,p}
    =
    \mathbf{W}_{\ell,s,p}\mathbf{x}_{\ell,s,p}
    +
    \mathbf{b}\,\delta_{\ell,0}\delta_{s,0}\delta_{p,1}.
\end{equation}
Here $\mathbf{W}_{\ell,s,p}$ mixes multiplicities within the same irrep, while the bias is applied only to invariant scalar channels. In practice, identical irrep blocks are grouped in memory so that the blockwise mixing can be evaluated efficiently with batched matrix multiplication.

\textbf{Equivariant Layer Normalization.} Following~\citep{equiformer}, the normalization layer treats scalar and non-scalar irreps separately. For scalar blocks with $(\ell,s)=(0,0)$, we subtract the mean over multiplicities before normalization. For all other blocks, no centering is applied. Each block is then normalized by an invariant RMS computed over multiplicities and irrep components. Learnable gains scale each multiplicity channel and are shared across irrep components. Learnable biases are restricted to scalar channels.

\section{Experiments}
To evaluate the expressivity of the SWSH tensor product, we benchmark our framework across four datasets covering five distinct tasks, ranging from synthetic geometric classification to large-scale atomistic prediction. We first use Chiral Tetris Classification to test whether spin-weighted channels can distinguish enantiomers that scalar Gaunt-based methods cannot separate. We then evaluate energy and force prediction on the 3BPA dataset, which tests generalization across thermally sampled conformations. To assess chiral molecular geometries more directly, we evaluate chirality classification and energy-force prediction on a chiral subset of the SPICE-MACE-OFF dataset. Finally, we integrate the SWSH tensor product into the Equiformer architecture and benchmark it on the Open Catalyst  (OC20) IS2RE dataset. Details of datasets, implementations, and comparisons are provided below. Furthermore, to evaluate efficiency, we show the time ablation of our implementations in~\Cref{sec:runtime_comparison}.

\subsection{Chiral Tetris Classification}

\begin{wrapfigure}[23]{r}{0.5\textwidth}
    \centering\vspace{-0.6cm}
    \vspace{-0.2in}
    \includegraphics[width=0.5\textwidth]{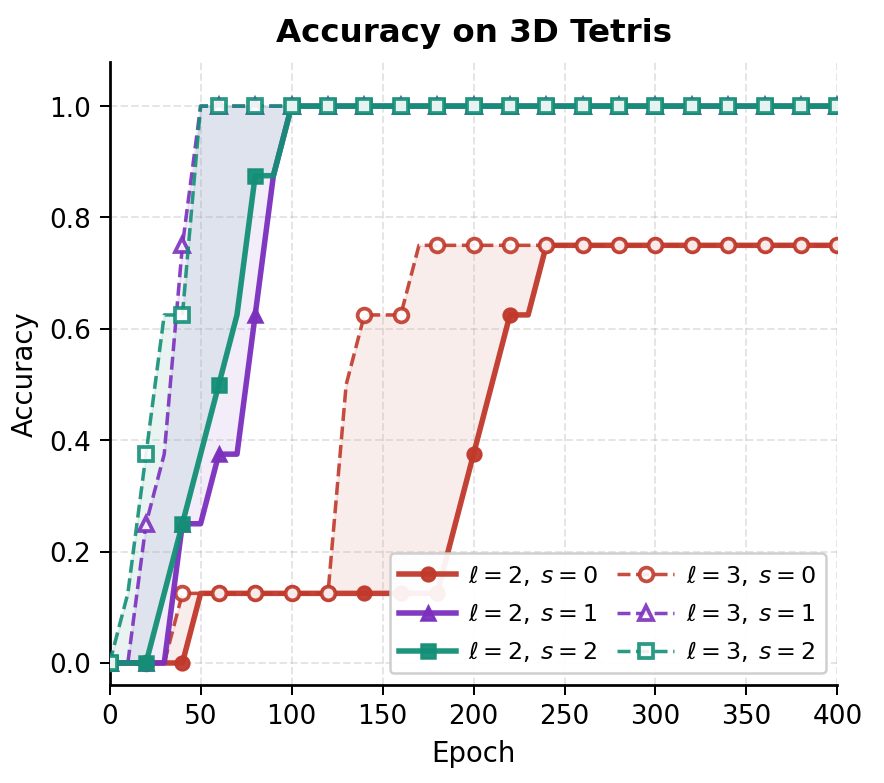}\vspace{-0.4cm}
    \caption{\textbf{3D Tetris classification.} Scalar GTP ($s=0$) plateaus at 75\%, consistent with its failure to distinguish the chiral mirror pair. In contrast, the SpinGTP with nonzero spin channels reaches 100\% accuracy and separates the two enantiomeric pieces, showing that spin-weighted coupling provides the missing odd $\ell$-sum interaction.}
    \label{fig:tetris}
\end{wrapfigure}

\textbf{Dataset}. We first use a minimal synthetic benchmark to test whether the model can represent chiral distinctions. Following~\citep{xie2024price,e3nn_software}, the task is to classify eight 3D Tetris-like shapes built from four unit cubes, where the first two classes are non-superimposable mirror images. Each shape is represented as a graph whose nodes are the cube centers, with edges connecting immediately adjacent cubes, and the inputs are presented under random 3D rotations. This makes the benchmark a direct probe of parity-sensitive expressivity rather than a generic shape-classification task.

\textbf{Training Details}. We train an equivariant network with four SpinGTP convolution layers, scalar node inputs, and SWSH edge features on radius-graph edges with cutoff \(1.5\). The model outputs one odd scalar channel for the mirror pair and six even scalar channels for the other classes, and is trained for 500 Adam steps with fresh random rotations, learning rate \(10^{-3}\), and MSE loss. 

\textbf{Results}. The key question is whether the model can separate the mirror pair. \citep{xie2024price} showed that CGTP-based networks solve this task, while scalar GTP fails because it lacks the required antisymmetric interaction. In our experiments, SpinGTP reaches 100\% accuracy and correctly distinguishes the two chiral pieces. This shows that the spin-weighted construction provides the missing odd $\ell$-sum interaction that scalar GTP cannot express. The Tetris results are shown in Figure~\ref{fig:tetris}.

\subsection{3BPA Performance}

\textbf{Dataset.} The 3BPA dataset~\citep{3BPA} consists of molecular dynamics trajectories of 3-(benzyloxy)pyridin-2-amine, a flexible drug-like molecule with three rotatable bonds that give rise to a complex torsional potential energy surface with many local minima. Following~\citep{batatia2022mace, batatia2025design, luo2024enabling}, we train on 500 geometries sampled at 300\,K. The benchmark evaluates both in-distribution generalization at 300\,K and out-of-distribution robustness at 600\,K and 1200\,K. A dihedral test set further probes the torsional potential energy surface by scanning one dihedral angle while holding the other two fixed, directly testing whether the model correctly resolves conformer transition barriers.

\textbf{Training Details.} We adopt MACE~\citep{batatia2022mace} as our base
architecture, extended with SWSH-based tensor products, which we refer to as SpinGTP. Training follows a two-stage scheme, with the second stage fine-tuning at an increased relative energy weight. Full details are provided in~\Cref{sec:3BPA training}.

\textbf{Results.}
We compare SpinGTP against Allegro~\citep{musaelian2023learning},
NequIP~\citep{nequip}, BOTNet~\citep{batatia2025design},
MACE~\citep{batatia2022mace}, and MACE-Gaunt~\citep{luo2024enabling} in Table~\ref{tab:3bpa_results}. SpinGTP consistently matches or outperforms current state-of-the-art methods across all test splits. In particular, SpinGTP maintains competitive in-distribution performance at 300\,K while showing stronger generalization on the out-of-distribution splits at 600\,K, 1200\,K, and the dihedral torsion test set. We additionally observe that MACE-Gaunt, despite sharing the same Gaunt-based contraction scheme, underperforms SpinGTP across all splits, with the largest gap on the dihedral torsion test set. This is consistent with our hypothesis that the antisymmetric spin-paths  restored by the SWSH tensor product contribute to resolving asymmetric torsional interactions, which are absent in the standard Gaunt contraction used by MACE-Gaunt.

\begin{table*}[t]
\centering
\caption{Results on the 3BPA dataset. Root mean square error of Energy ($E$) in meV and Forces ($F$) in meV/\AA\ across three temperature regimes and a dihedral scan. Standard deviations over 3 random seeds are in parentheses. \textbf{Bold} indicates best performance, \underline{underline} indicates second best.}
\label{tab:3bpa_results}
\vspace{0.1in}
\begin{small}
\begin{tabular}{llcccccc}
\toprule
& & \textbf{Allegro} & \textbf{NequIP} & \textbf{BOTNet} & \textbf{MACE} & \textbf{MACE-Gaunt} & \textbf{SpinGTP} \\
\midrule
\multirow{2}{*}{300 K}
 & $E$ & 3.84 (0.08)  & 3.3 (0.1)  & 3.1 (0.13) & 3.0 (0.2)  & \textbf{2.9 (0.1)}  & \textbf{2.9 (0.1)} \\
 & $F$ & 12.98 (0.17) & 10.8 (0.2) & 11.0 (0.14)& \textbf{8.8 (0.3)}  & 9.2 (0.1)  & \underline{9.0 (0.3)} \\
\midrule
\multirow{2}{*}{600 K}
 & $E$ & 12.07 (0.45) & 11.2 (0.1) & 11.5 (0.6) & \underline{9.7 (0.5)}  & 10.6 (0.5) & \textbf{9.6 (0.3)} \\
 & $F$ & 29.17 (0.22) & 26.4 (0.1) & 26.7 (0.29)& \underline{21.8 (0.6)} & 22.2 (0.2) & \textbf{21.0 (0.5)} \\
\midrule
\multirow{2}{*}{1200 K}
 & $E$ & 42.57 (1.46) & 38.5 (1.6) & 39.1 (1.1) & \underline{29.8 (1.0)} & 30.4 (1.2) & \textbf{29.0 (1.1)} \\
 & $F$ & 82.96 (1.77) & 76.2 (1.1) & 81.1 (1.5) & \underline{62.0 (0.7)} & 63.1 (1.2) & \textbf{61.4 (3.2)} \\
\midrule
\multirow{2}{*}{Dihedral}
 & $E$ & --           & --         & 16.3 (1.5) & \textbf{7.8 (0.6)}  & 9.9 (0.3)  & \underline{8.9 (2.5)} \\
 & $F$ & --           & --         & 20.0 (1.2) & 16.5 (1.7) & 17.7 (1.1) & \textbf{15.3 (0.7)} \\
\midrule
\end{tabular}
\end{small}
\vspace{-10pt}
\end{table*}

\subsection{SPICE-MACE-OFF Chiral Subset Performance}

\textbf{Dataset.} The SPICE-MACE-OFF dataset~\citep{kovacs2025mace} is a large-scale benchmark comprising over 950,000 unique configurations, including small molecules from PubChem, DES370K dimers, and biological systems. To evaluate the specific impact of the SWSH tensor product on chiral expressivity, we curated a chiral subset by screening the original dataset using RDKit~\cite{rdkit} to identify enantiomeric pairs and chiral centers via SMILES sequences. This allows us to probe whether the addition of antisymmetric spin-paths
translates to improved performance on geometrically sensitive tasks.

\textbf{Experimental Settings.} We evaluate our framework on two distinct categories of tasks. First, we perform chirality classification by modifying the regression heads of  \cite{nequip}, MACE \cite{batatia2022mace} and MACE-Gaunt \cite{luo2024enabling} to distinguish between R- and S-chirality. This includes both classification using mixed R/S training data and a parity generalization test, where the model is trained exclusively on R-chirality and evaluated on unseen S-chirality structures. Second, we assess Machine Learning Interatomic Potentials (MLIP) through standard energy and force prediction across four chiral sub-domains: PubChem, Amino Acids, DES370, and Dipeptides.

\begin{table*}[t]
\centering
\caption{Chirality classification (R vs.\ S, both in training). Validation accuracy (\%) at selected epochs. Values are reported as mean (std) over 3 runs. Convergence: epochs to first reach 98\% validation accuracy (lower is better), also mean (std) over 3 runs.}
\label{tab:chirality_classification}
\vspace{0.1in}
\begin{small}
\begin{tabular}{lccccc}
\toprule
Model & Epoch 1 $\uparrow$ & Epoch 10 $\uparrow$ & Epoch 20 $\uparrow$ & Best (\%) $\uparrow$ & Convergence $\downarrow$ \\
\midrule
MACE-Gaunt & 64.9 (0.5) & 95.5 (0.3) & 95.8 (0.3) & 98.30 (0.08) & 24.3 (1.0) \\
MACE & 67.8 (0.4) & 96.8 (0.3) & 97.3 (0.2) & 98.34 (0.07) & 18.2 (0.8) \\
NequIP & 68.1 (0.4) & 97.0 (0.2) & 97.6 (0.2) & 98.43 (0.07) & 16.1 (0.6) \\
\textbf{SpinGTP} & 67.4 (0.3) & \textbf{97.5 (0.2)} & \textbf{98.4 (0.1)} & \textbf{98.79 (0.05)} & \textbf{10.9 (0.5)} \\
\bottomrule
\end{tabular}
\end{small}
\vspace{-10pt}
\end{table*}

\textbf{Chirality Classification and Parity Generalization.}

Results for the chirality classification tasks are summarized in Tables~\ref{tab:chirality_classification}. In the standard classification task, SpinGTP achieves the fastest convergence, reaching 98\% validation accuracy in nearly half the epochs required by MACE-Gaunt ($10.9 \pm 0.5$ vs. $24.3 \pm 1.0$). The parity generalization test in Table~\ref{tab:chirality_parity} reveals a critical failure mode in the original Gaunt tensor product (MACE-Gaunt), which fails entirely ($0\%$ accuracy) to generalize to unseen S-chirality. This stems from the lack of odd-parity paths in standard Gaunt coefficients. In contrast, SpinGTP not only enables this generalization by encoding proper reflection symmetry but also outperforms Clebsch-Gordan (CG) based models, achieving a superior $94.1\%$ accuracy on unseen mirror structures.

\begin{wraptable}{r}{0.55\textwidth}
\vspace{0.15in}
\vspace{-\intextsep}
\centering
\caption{Chirality parity generalization: train on R-only, evaluate on R-only and S-only. acc\_S tests generalization to mirror structures never seen in training. Values are mean (std) over 3 runs.}
\label{tab:chirality_parity}
\vspace{0.05in}
\begin{small}
\begin{tabular}{lccc}
\toprule
Model & acc\_R (\%) & acc\_S (\%) & Overall (\%) \\
\midrule
MACE-Gaunt       & 100.0 (0.0) & 0.0 (0.0)  & 46.3 (0.0) \\
MACE             & 100.0 (0.0) & 88.6 (1.2) & 93.9 (0.7) \\
NequIP           & 100.0 (0.0) & 90.8 (1.0) & 95.0 (0.6) \\
\textbf{SpinGTP} & \textbf{100.0 (0.0)} & \textbf{94.1 (0.8)} & \textbf{97.2 (0.4)} \\
\bottomrule
\end{tabular}
\end{small}
\vspace{0.1in}
\vspace{-\intextsep}
\end{wraptable}

\textbf{Energy and Force Prediction.}
Table~\ref{tab:chiral_mlip} demonstrates that the completeness of the SWSH tensor product translates to higher accuracy in MLIP tasks. Compared to MACE-Gaunt, our model provides a reduction in both energy and force errors across most of the chiral subsets. Notably, SWSH achieves state-of-the-art performance in the Chiral DES370 and Dipeptides subsets, where sensitivity to complex torsional and chiral interactions is paramount. Even when compared to the computationally heavier Clebsch-Gordan implementations in NequIP and MACE, SpinGTP remains highly competitive, consistently achieving the lowest errors.

\textbf{Training Setup.}
Detailed specifications regarding architecture, hyperparameters, and training protocols are provided in~\Cref{sec:spice training}.


\begin{table*}[htbp]
\centering
\caption{{Results on SPICE-MACE-OFF Chiral Subset.} Test set MAE for models trained and evaluated exclusively on SPICE Chiral sub-datasets. Energy ($E$) in meV/atom, Forces ($F$) in meV/\AA. The best results are shown in \textbf{bold} and the second best results are shown with \underline{underlines}.}
\label{tab:chiral_mlip}
\vspace{0.1in}
\begin{small}
\begin{tabular}{lcccccccc}
\toprule
& \multicolumn{2}{c}{\textbf{NequIP}} & \multicolumn{2}{c}{\textbf{MACE}} & \multicolumn{2}{c}{\textbf{MACE-Gaunt}} & \multicolumn{2}{c}{\textbf{SpinGTP}} \\
\cmidrule(lr){2-3} \cmidrule(lr){4-5} \cmidrule(lr){6-7} \cmidrule(lr){8-9}
\textbf{Chiral Subset} & $E$ & $F$ & $E$ & $F$ & $E$ & $F$ & $E$ & $F$ \\
\midrule
Chiral PubChem     & 6.1 & 21.0 & \underline{5.0} & \textbf{16.0} & \textbf{4.9} & 19.3 & 5.7 & \underline{17.6} \\
Chiral Amino Acids & \textbf{6.5} & 20.5 & 7.5 & \textbf{19.3} & \underline{6.6} & 22.9 & \underline{6.6} & \underline{19.9} \\
Chiral DES370      & 3.3 & 8.4  & \underline{2.6} & \underline{6.6} & 2.9 & 8.1  & \textbf{2.4} & \textbf{6.3} \\
Chiral Dipeptides  & 4.0 & 12.0 & 4.2 & \underline{9.7} & \underline{3.5} & 11.8 & \textbf{3.2} & \textbf{9.4} \\
\bottomrule
\end{tabular}
\end{small}
\end{table*}

\subsection{OC20 IS2RE Direct}


\textbf{Dataset}. We test whether the SpinGTP implementation remains compatible with large-scale atomistic training. We use the Open Catalyst 2020 (OC20) initial-structure-to-relaxed-energy (IS2RE) task~\citep{chanussot2021open}, where each system contains an adsorbate on a catalyst slab and the goal is to predict the relaxed energy from the initial structure. These graphs are substantially larger and chemically more diverse than 3BPA and SPICE. Following standard OC20 reporting, validation sets are divided into four sub-splits, ID, OOD-Ads, OOD-Cat, and OOD-Both. The direct IS2RE split contains 460k training structures and 100k structures for validation.

\textbf{Training Details}. For the main comparison, we use the direct IS2RE setting without the IS2RS node-level auxiliary task, following the comparison protocol used in Equiformer. Our SpinGTP follows the Equiformer attention backbone and replaces its equivariant tensor-product, linear, and normalization blocks with SWSH counterparts. We use a six-layer SWSH-Equiformer with eight attention heads, cutoff radius \(5.0\), 128 radial basis functions, on-the-fly periodic graphs, and a maximum of 500 neighbors. Additional architecture and optimization details are provided in~\Cref{sec:OC20 training}.

\textbf{Results}. \Cref{tab:is2re} compares our model against direct IS2RE baselines, including SchNet, DimeNet++, GemNet-dT, SphereNet, Equiformer, and EquiformerV2. Our model achieves performance comparable to Equiformer, which uses Clebsch-Gordan tensor products. In particular, SpinGTP improves the ID split, suggesting that the spin-weighted tensor product retains strong expressive power while remaining compatible with large-scale OC20 training.

\begin{table}[t]
\centering
\caption{Comparison of model performance on energy predictions for OC20 \textsc{IS2RE-Direct} validation set without noisy-node auxiliary loss. Our model is trained to compare against several baseline methods, including SchNet~\citep{SchNet}, DimeNet++~\citep{DimeNet++},
GemNet-dT~\citep{GemNet-dT},
SphereNet~\citep{SphereNet}, 
ComENet~\citep{comenet}, Equiformer~\citep{equiformer} and EquiformerV2~\citep{equiformerv2}. The best results are shown in \textbf{bold} and the second best results are shown with \underline{underlines}.}
    \vspace{0.01in}
    \label{tab:is2re}
\begin{sc}
\resizebox{\textwidth}{!}{
\begin{tabular}{lccccc|ccccc}
\toprule[1.2pt]
 & \multicolumn{5}{c|}{Energy MAE ({\normalfont eV}) $\downarrow$} & \multicolumn{5}{c}{EwT (\%) $\uparrow$} \\ 
\cmidrule[0.6pt]{2-11}
Model & ID & OOD Ads & OOD Cat & OOD Both & Average & ID & OOD Ads & OOD Cat & OOD Both & Average \\
\midrule[1.2pt]
SchNet & 0.6465 & 0.7074 & 0.6475 & 0.6626 & 0.6660 & 2.96 & 2.22 & 3.03 & 2.38 & 2.65 \\
DimeNet++ & 0.5636 & 0.7127 & 0.5612 & 0.6492 & 0.6217 & 4.25 & 2.48 & 4.40 & 2.56 & 3.42 \\
GemNet-dT & 0.5561 & 0.7342 & 0.5659 & 0.6964 & 0.6382 & 4.51 & 2.24 & 4.37 & 2.38 & 3.38 \\
SphereNet & 0.5632 & 0.6682 & 0.5590 & 0.6190 & 0.6024 & 4.56 & \underline{2.70} & 4.59 & 2.70 & 3.64\\
Equiformer & \underline{0.5088} & \textbf{0.6271} & \textbf{0.5051} & \textbf{0.5545} & \textbf{0.5489} & \underline{4.88} & \textbf{2.93} & 4.92 & \textbf{2.98} & \textbf{3.93} \\
EquiformerV2 & 0.5161 &  0.7041 & 0.5245 & 0.6365 & 0.5953 & - & - & - & - & - \\
SpinGTP & \textbf{0.5066} & \underline{0.6510} & \underline{0.5130} & \underline{0.5887} & \underline{0.5648} & \textbf{5.07} & 2.58 & \textbf{5.06} & \underline{2.85} & \underline{3.89} \\
\bottomrule[1.2pt]
\end{tabular}
}
\end{sc}
\vspace{-0.2in}
\end{table}


\section{Limitations and Summary}
\label{sec:summary}

{\bf{Limitations}}. One limitation of SpinGTP is that nonzero spin-weighted features require a consistent gauge. Without compatible frame choices, aggregating spin features can violate equivariance. We address this through geometry-dependent frame construction with a shared convention across nonzero spin irreps (see~\Cref{sec:equiv}), though ambiguity remains for highly symmetric point clouds. Another limitation is that we did not observe a clear acceleration from using the spherical transform, possibly because the multiplicity is large relative to the small $L$. Resolving these limitations is a direction for future work.

{\bf{Summary}}. In this work, we introduced a novel equivariant framework based on Spin-Weighted Spherical Harmonics (SWSH), bridging the gap between the computational efficiency of Gaunt-based tensor products and the mathematical completeness of the Clebsch-Gordan Tensor Product. By leveraging the spin-weight degree of freedom, we recover the antisymmetric paths essential for resolving chiral geometries. Theoretically, we demonstrate that the SWSH basis captures odd-parity interactions through specialized spin-selection rules. Architecturally, we proposed equivariant SWSH linear and normalization layers that enforce strict symmetry constraints while optimizing for hardware throughput. Empirically, our framework demonstrates superior expressivity across diverse benchmarks. These results suggest that SWSH-based equivariant networks offer a robust and scalable solution for high-fidelity geometric modeling in molecular science and beyond.

\section*{Acknowledgments}

This work was supported in part by the National Science Foundation under Grants IIS-2243850, CNS-2328395, and MOMS-2331036; the National Institutes of Health under Grant U01AG070112; the Texas A\&M University Division of Research Targeted Proposal Teams Funding Program; and the Texas A\&M Institute of Data Science Thematic Labs Program.

\bibliographystyle{refstyle}
\bibliography{ref}

@article{reisswig2013general,
  title={General relativistic null-cone evolutions with a high-order scheme},
  author={Reisswig, Christian and Bishop, Nigel and Pollney, Denis},
  journal={General Relativity and Gravitation},
  volume={45},
  number={5},
  pages={1069--1094},
  year={2013},
  publisher={Springer}
}

@article{healy2003ffts,
  title={FFTs for the 2-sphere-improvements and variations},
  author={Healy Jr, Dennis M and Rockmore, Daniel N and Kostelec, Peter J and Moore, Sean},
  journal={Journal of Fourier analysis and applications},
  volume={9},
  number={4},
  pages={341--385},
  year={2003},
  publisher={Springer}
}

@article{kovacs2025mace,
  title={Mace-off: Short-range transferable machine learning force fields for organic molecules},
  author={Kov{\'a}cs, D{\'a}vid P{\'e}ter and Moore, J Harry and Browning, Nicholas J and Batatia, Ilyes and Horton, Joshua T and Pu, Yixuan and Kapil, Venkat and Witt, William C and Magdau, Ioan-Bogdan and Cole, Daniel J and others},
  journal={Journal of the American Chemical Society},
  volume={147},
  number={21},
  pages={17598--17611},
  year={2025},
  publisher={ACS Publications}
}

@article{batatia2025design,
  title={The design space of E (3)-equivariant atom-centred interatomic potentials},
  author={Batatia, Ilyes and Batzner, Simon and Kov{\'a}cs, D{\'a}vid P{\'e}ter and Musaelian, Albert and Simm, Gregor NC and Drautz, Ralf and Ortner, Christoph and Kozinsky, Boris and Cs{\'a}nyi, G{\'a}bor},
  journal={Nature Machine Intelligence},
  volume={7},
  number={1},
  pages={56},
  year={2025}
}

@article{batatia2022mace,
  title={MACE: Higher order equivariant message passing neural networks for fast and accurate force fields},
  author={Batatia, Ilyes and Kovacs, David P and Simm, Gregor and Ortner, Christoph and Cs{\'a}nyi, G{\'a}bor},
  journal={Advances in neural information processing systems},
  volume={35},
  pages={11423--11436},
  year={2022}
}

@article{3BPA,
    author = {Kovács, Dávid P{\'e}ter and Oord, Cas van der and Kucera, Jiri and Allen, Alice E. A. and Cole, Daniel J. and Ortner, Christoph and Csányi, Gábor},
    title = {Linear Atomic Cluster Expansion Force Fields for Organic Molecules: Beyond RMSE},
    journal = {Journal of Chemical Theory and Computation},
    volume = {17},
    number = {12},
    pages = {7696-7711},
    year = {2021},
    doi = {10.1021/acs.jctc.1c00647},
        note ={PMID: 34735161},
    URL = { 
        https://doi.org/10.1021/acs.jctc.1c00647
    },
    eprint = { 
        https://doi.org/10.1021/acs.jctc.1c00647
    }
}

@inproceedings{luo2024enabling,
  title={Enabling Efficient Equivariant Operations in the Fourier Basis via Gaunt Tensor Products},
  author={Luo, Shengjie and Chen, Tianlang and Krishnapriyan, Aditi S},
  booktitle={The Twelfth International Conference on Learning Representations},
  year={2024}
}

@article{musaelian2023learning,
  title={Learning local equivariant representations for large-scale atomistic dynamics},
  author={Musaelian, Albert and Batzner, Simon and Johansson, Anders and Sun, Lixin and Owen, Cameron J and Kornbluth, Mordechai and Kozinsky, Boris},
  journal={Nature Communications},
  volume={14},
  number={1},
  pages={579},
  year={2023},
  publisher={Nature Publishing Group UK London}
}

@article{goldberg1967spin,
  title={Spin-s Spherical Harmonics and {\dh}},
  author={Goldberg, Joshua N and MacFarlane, Alan J and Newman, Ezra T and Rohrlich, Fritz and Sudarshan, EC George},
  journal={Journal of Mathematical Physics},
  volume={8},
  number={11},
  pages={2155--2161},
  year={1967},
  publisher={AIP Publishing}
}

@article{xie2026asymptotically,
  title={Asymptotically Fast Clebsch-Gordan Tensor Products with Vector Spherical Harmonics},
  author={Xie, YuQing and Daigavane, Ameya and Kotak, Mit and Smidt, Tess},
  journal={arXiv preprint arXiv:2602.21466},
  year={2026}
}

@inproceedings{xie2024price,
  title={The Price of Freedom: Exploring Expressivity and Runtime Tradeoffs in Equivariant Tensor Products},
  author={Xie, Yuqing and Daigavane, Ameya and Kotak, Mit and Smidt, Tess},
  booktitle={International Conference on Machine Learning},
  pages={68599--68625},
  year={2025},
  organization={PMLR}
}

@book{khersonskii1988quantum,
  title={Quantum theory of angular momemtum},
  author={Khersonskii, Valery Kelmanovich and Moskalev, Anatoly Nikolaevich and Varshalovich, Dmitry Alexandrovich},
  year={1988},
  publisher={World Scientific}
}

@inproceedings{passaro2023reducing,
  title={Reducing SO (3) convolutions to SO (2) for efficient equivariant GNNs},
  author={Passaro, Saro and Zitnick, C Lawrence},
  booktitle={International conference on machine learning},
  pages={27420--27438},
  year={2023},
  organization={PMLR}
}

@inproceedings{
lin2025tensor,
title={Tensor Decomposition Networks for Accelerating Machine Learning Force Field Computations},
author={Yuchao Lin and Cong Fu and Zachary Krueger and Haiyang Yu and Maho Nakata and Jianwen Xie and Emine Kucukbenli and Xiaofeng Qian and Shuiwang Ji},
booktitle={The Thirty-ninth Annual Conference on Neural Information Processing Systems},
year={2025},
url={https://openreview.net/forum?id=9vKJyCUfMH}
}

@article{zhang2025artificial,
  title={Artificial intelligence for science in quantum, atomistic, and continuum systems},
  author={Zhang, Xuan and Wang, Limei and Helwig, Jacob and Luo, Youzhi and Fu, Cong and Xie, Yaochen and Liu, Meng and Lin, Yuchao and Xu, Zhao and Yan, Keqiang and others},
  journal={Foundations and Trends{\textregistered} in Machine Learning},
  volume={18},
  number={4},
  pages={385--849},
  year={2025},
  publisher={Emerald Publishing Limited}
}

@article{bronstein2021geometric,
  title={Geometric deep learning: Grids, groups, graphs, geodesics, and gauges},
  author={Bronstein, Michael M and Bruna, Joan and Cohen, Taco and Veli{\v{c}}kovi{\'c}, Petar},
  journal={arXiv preprint arXiv:2104.13478},
  year={2021}
}

@article{villar2021scalars,
  title={Scalars are universal: Equivariant machine learning, structured like classical physics},
  author={Villar, Soledad and Hogg, David W and Storey-Fisher, Kate and Yao, Weichi and Blum-Smith, Ben},
  journal={Advances in neural information processing systems},
  volume={34},
  pages={28848--28863},
  year={2021}
}

@article{kondor2025principles,
  title={The principles behind equivariant neural networks for physics and chemistry},
  author={Kondor, Risi},
  journal={Proceedings of the National Academy of Sciences},
  volume={122},
  number={41},
  pages={e2415656122},
  year={2025},
  publisher={National Academy of Sciences}
}

@article{fei2024rotation,
  title={Rotation invariance and equivariance in 3D deep learning: a survey},
  author={Fei, Jiajun and Deng, Zhidong},
  journal={Artificial Intelligence Review},
  volume={57},
  number={7},
  pages={168},
  year={2024},
  publisher={Springer}
}

@article{anderson2019cormorant,
  title={Cormorant: Covariant molecular neural networks},
  author={Anderson, Brandon and Hy, Truong Son and Kondor, Risi},
  journal={Advances in neural information processing systems},
  volume={32},
  year={2019}
}

@article{thomas2018tensor,
  title={Tensor field networks: Rotation-and translation-equivariant neural networks for 3d point clouds},
  author={Thomas, Nathaniel and Smidt, Tess and Kearnes, Steven and Yang, Lusann and Li, Li and Kohlhoff, Kai and Riley, Patrick},
  journal={arXiv preprint arXiv:1802.08219},
  year={2018}
}

@article{chanussot2021open,
  title={Open catalyst 2020 (OC20) dataset and community challenges},
  author={Chanussot, Lowik and Das, Abhishek and Goyal, Siddharth and Lavril, Thibaut and Shuaibi, Muhammed and Riviere, Morgane and Tran, Kevin and Heras-Domingo, Javier and Ho, Caleb and Hu, Weihua and others},
  journal={Acs Catalysis},
  volume={11},
  number={10},
  pages={6059--6072},
  year={2021},
  publisher={ACS Publications}
}

@article{SchNet,
  title={Schnet: A continuous-filter convolutional neural network for modeling quantum interactions},
  author={Sch{\"u}tt, Kristof and Kindermans, Pieter-Jan and Sauceda Felix, Huziel Enoc and Chmiela, Stefan and Tkatchenko, Alexandre and M{\"u}ller, Klaus-Robert},
  journal={Advances in neural information processing systems},
  volume={30},
  year={2017}
}

@article{DimeNet++,
  title={Fast and uncertainty-aware directional message passing for non-equilibrium molecules},
  author={Gasteiger, Johannes and Giri, Shankari and Margraf, Johannes T and G{\"u}nnemann, Stephan},
  journal={arXiv preprint arXiv:2011.14115},
  year={2020}
}

@article{GemNet-dT,
  title={Gemnet: Universal directional graph neural networks for molecules},
  author={Gasteiger, Johannes and Becker, Florian and G{\"u}nnemann, Stephan},
  journal={Advances in neural information processing systems},
  volume={34},
  pages={6790--6802},
  year={2021}
}

@inproceedings{SphereNet,
  title={Spherical message passing for {3D} molecular graphs},
  author={Liu, Yi and Wang, Limei and Liu, Meng and Lin, Yuchao and Zhang, Xuan and Oztekin, Bora and Ji, Shuiwang},
  booktitle={International Conference on Learning Representations},
  year={2022}
}

@article{comenet,
  title={ComENet: Towards complete and efficient message passing for 3D molecular graphs},
  author={Wang, Limei and Liu, Yi and Lin, Yuchao and Liu, Haoran and Ji, Shuiwang},
  journal={Advances in Neural Information Processing Systems},
  volume={35},
  pages={650--664},
  year={2022}
}

@inproceedings{
equiformer,
title={Equiformer: Equivariant Graph Attention Transformer for 3D Atomistic Graphs},
author={Yi-Lun Liao and Tess Smidt},
booktitle={The Eleventh International Conference on Learning Representations },
year={2023},
url={https://openreview.net/forum?id=KwmPfARgOTD}
}

@inproceedings{
equiformerv2,
title={EquiformerV2: Improved Equivariant Transformer for Scaling to Higher-Degree Representations},
author={Yi-Lun Liao and Brandon M Wood and Abhishek Das and Tess Smidt},
booktitle={The Twelfth International Conference on Learning Representations},
year={2024},
url={https://openreview.net/forum?id=mCOBKZmrzD}
}

@software{rdkit,
  author       = {Greg Landrum and
                  Paolo Tosco and
                  Brian Kelley and
                  Ricardo Rodriguez and
                  David Cosgrove and
                  Riccardo Vianello and
                  sriniker and
                  Peter Gedeck and
                  Gareth Jones and
                  Eisuke Kawashima and
                  NadineSchneider and
                  Dan Nealschneider and
                  tadhurst-cdd and
                  Andrew Dalke and
                  Matt Swain and
                  Brian Cole and
                  Samo Turk and
                  Aleksandr Savelev and
                  Niels Maeder and
                  Yakov Pechersky and
                  Alain Vaucher and
                  Maciej Wójcikowski and
                  Rachel Walker and
                  Hussein Faara and
                  Ichiru Take and
                  Vincent F. Scalfani and
                  Daniel Probst and
                  Kazuya Ujihara and
                  Jeremy Monat and
                  Juuso Lehtivarjo},
  title        = {rdkit/rdkit: 2026\_03\_2 (Q1 2026) Release},
  month        = apr,
  year         = 2026,
  publisher    = {Zenodo},
  version      = {Release\_2026\_03\_2},
  doi          = {10.5281/zenodo.19922430},
  url          = {https://doi.org/10.5281/zenodo.19922430},
}

@article{nequip,
  title={E (3)-equivariant graph neural networks for data-efficient and accurate interatomic potentials},
  author={Batzner, Simon and Musaelian, Albert and Sun, Lixin and Geiger, Mario and Mailoa, Jonathan P and Kornbluth, Mordechai and Molinari, Nicola and Smidt, Tess E and Kozinsky, Boris},
  journal={Nature communications},
  volume={13},
  number={1},
  pages={2453},
  year={2022},
  publisher={Nature Publishing Group UK London}
}

@inproceedings{
li2026eformer,
title={E2Former: An Efficient and Equivariant Transformer with Linear-Scaling Tensor Products},
author={Yunyang Li and Lin Huang and Zhihao Ding and Xinran Wei and Chu Wang and Han Yang and Zun Wang and Chang Liu and Yu Shi and Peiran Jin and Tao Qin and Mark Gerstein and Jia Zhang},
booktitle={The Thirty-ninth Annual Conference on Neural Information Processing Systems},
year={2026},
url={https://openreview.net/forum?id=ls5L4IMEwt}
}

@inproceedings{schutt2021equivariant,
  title={Equivariant message passing for the prediction of tensorial properties and molecular spectra},
  author={Sch{\"u}tt, Kristof and Unke, Oliver and Gastegger, Michael},
  booktitle={International conference on machine learning},
  pages={9377--9388},
  year={2021},
  organization={PMLR}
}

@article{haghighatlari2022newtonnet,
  title={NewtonNet: a Newtonian message passing network for deep learning of interatomic potentials and forces},
  author={Haghighatlari, Mojtaba and Li, Jie and Guan, Xingyi and Zhang, Oufan and Das, Akshaya and Stein, Christopher J and Heidar-Zadeh, Farnaz and Liu, Meili and Head-Gordon, Martin and Bertels, Luke and others},
  journal={Digital Discovery},
  volume={1},
  number={3},
  pages={333--343},
  year={2022},
  publisher={Royal Society of Chemistry}
}

@software{e3nn_software,
    author = {Mario Geiger and
              Tess Smidt and
              Alby M. and
              Benjamin Kurt Miller and
              Wouter Boomsma and
              Bradley Dice and
              Kostiantyn Lapchevskyi and
              Maurice Weiler and
              Michał Tyszkiewicz and
              Simon Batzner and
              Dylan Madisetti and
              Martin Uhrin and
              Jes Frellsen and
              Nuri Jung and
              Sophia Sanborn and
              Mingjian Wen and
              Josh Rackers and
              Marcel Rød and
              Michael Bailey},
    title = {Euclidean neural networks: e3nn},
    month = apr,
    year = 2022,
    publisher = {Zenodo},
    version = {0.5.0},
    doi = {10.5281/zenodo.6459381},
    url = {https://doi.org/10.5281/zenodo.6459381}
}

@article{fu2025augmenting,
  title={Augmenting Molecular Graphs with Geometries via Machine Learning Interatomic Potentials},
  author={Fu, Cong and Lin, Yuchao and Krueger, Zachary and Yu, Haiyang and Nakata, Maho and Xie, Jianwen and Kucukbenli, Emine and Qian, Xiaofeng and Ji, Shuiwang},
  journal={arXiv preprint arXiv:2507.00407},
  year={2025}
}

\newpage
\appendix

\section{Spin-Weighted Spherical Harmonics (SWSH)}
\label[appsec]{sec:swsh}

Spin-weighted spherical harmonics (SWSH), denoted as ${}_sY_{\ell m}(\theta, \phi)$, are generalizations of the standard spherical harmonics $Y_{\ell m}$. Unlike ordinary spherical harmonics, which are scalar fields, spin-weighted harmonics are functions on the sphere that behave as $\mathrm{U}(1)$ gauge fields, characterized by a degree $\ell$, a magnetic order $m$, and a spin weight $s$ satisfying $|s| \leq \ell$.

In the standard case where $s=0$, these functions reduce to the conventional scalar spherical harmonics, such that ${}_0Y_{\ell m} = Y_{\ell m}$. Like their scalar counterparts, SWSHs form a complete orthonormal basis over the sphere $\mathrm{S}^2$, satisfying the orthogonality condition:
\begin{equation}
    \int_{S^{2}} {}_sY_{\ell m}\, {}_s{Y}_{\ell' m'}^*\, d\Omega = \delta_{\ell\ell'}\delta_{mm'}.
\end{equation}
This orthonormality, combined with their unique transformation properties under local rotations, makes them ideal for representing high-dimensional equivariant features that standard scalar bases cannot fully capture.

\subsection{Explicit Formula}
\label[appsec]{sec:swsh_formula}

The spin-weighted spherical harmonics can be calculated directly using the formula~\citep{goldberg1967spin}
\begin{equation}
\begin{split}
{}_sY_{\ell m}(\theta, \phi) &= A_{s\ell m} \sin^{2\ell}\left({\frac{\theta}{2}}\right)e^{im\phi} \\
&\quad \times \sum_{r=0}^{\ell-s}\left(-1\right)^{r}\binom{\ell-s}{r}\binom{\ell+s}{r+s-m} \\
&\quad \times \cot^{2r+s-m}\left({\frac{\theta}{2}}\right).
\end{split}
\end{equation}

Specifically, the first few spin-weighted spherical harmonics for $s=1$ and $\ell=1$ are given by
\begin{align}
    {}_1Y_{10}(\theta, \phi) &= {\sqrt{\frac{3}{8\pi}}}\,\sin \theta, \quad
    {}_1Y_{1\pm 1}(\theta, \phi) = -{\sqrt{\frac{3}{16\pi}}}(1\mp \cos \theta)\,e^{\pm i\phi}.
\end{align}

With the phase convention used in this definition, the harmonics satisfy the following conjugation and parity relations
\begin{align}
    {}_s {Y}_{\ell m}^* &= (-1)^{s+m}{}_{-s}Y_{\ell(-m)}, \quad
    {}_sY_{\ell m}(\pi - \theta, \phi + \pi) = (-1)^{\ell}{}_{-s}Y_{\ell m}(\theta, \phi).
\end{align}

\subsection{Equivariance and Spin-Weighted Transformations}
\label[appsec]{sec:equiv}

A defining characteristic of spin-weighted spherical harmonics is their transformation law under the rotation group $\mathrm{SO}(3)$. Unlike standard scalar harmonics ($s=0$), which transform solely via the Wigner $D$-matrices, SWSHs are sections of a line bundle and thus pick up a local phase shift corresponding to the rotation of the local tangent frame.

\paragraph{Rotation Law.} 
Given a rotation $R \in \mathrm{SO}(3)$, the transformation of a spin-weighted signal at a point $\mathbf{x} \in \mathrm{S}^2$ is governed by:
\begin{equation}
    {}_sY_{\ell m}(R\mathbf{x}) = e^{is\psi(R, \mathbf{x})} \sum_{m'=-\ell}^{\ell} D^{(\ell)}_{m'm}(R) \, {}_sY_{\ell m'}(\mathbf{x})
\end{equation}
where $D^{(\ell)}_{m'm}(R)$ are the elements of the $(2\ell+1)$-dimensional irreducible representation of $\mathrm{SO}(3)$. The term $e^{is\psi}$ represents a $\mathrm{U}(1)$ gauge transformation, where $\psi(R, \mathbf{x})$ is the angle by which the local tangent frame at $\mathbf{x}$ rotates relative to the fixed coordinate basis after applying $R$.

\paragraph{Geometry-Dependent Frames.}

In our implementation, the gauge is fixed by equivariant frames constructed from the input geometry. Before evaluating SWSHs, the input is rotated into a common frame. We implement both the global frame and local node frame based on the geometry.

For the \textbf{global frame}, given a graph \(b\) and point cloud $\{\mathbf{x}_i\}_{i\in b}$, let
\[
    \mathbf c_b
    = \sum_{i\in b} w_i \mathbf x_i,
    \qquad
    \widetilde{\mathbf x}_i=\mathbf x_i-\mathbf c_b .
\]
We construct a graph frame \(F_b\in \mathrm{SO}(3)\) from the eigensystem of the
centered covariance
\begin{equation}
    C_b
    =
    \sum_{i\in b}
    w_i \widetilde{\mathbf x}_i\widetilde{\mathbf x}_i^\top .
\end{equation}
The eigenvector signs are fixed by parity-even axial references
\[
    \mathbf h_b^{(q)}
    =
    \sum_{i\in b}
    w_i \|\widetilde{\mathbf r}_i\|^q
    \widetilde{\mathbf r}_i,
    \qquad
    \mathbf a_b=\mathbf h_b^{(1)}\times \mathbf h_b^{(2)},
    \qquad
    \mathbf b_b=\mathbf h_b^{(2)}\times \mathbf h_b^{(3)} .
\]
If no repeated eigenvalues and vanishing axial references are present, this gives a
permutation-equivariant, \(\mathrm{SO}(3)\)-equivariant, and parity-invariant graph
frame.

If the global frame encounters eigenvector ambiguity, we use the \textbf{local node frame}. For each node \(j\), we similarly construct a local node frame \(F_j\in \mathrm{SO}(3)\)
from the neighborhood quadrupole
\begin{equation}
    Q_j
    =
    \sum_{k\to j}
    w_{jk}\left(
    \|\mathbf r_{jk}\|^2 I
    -
    \mathbf r_{jk}\mathbf r_{jk}^\top
    \right),
\end{equation}
where \(\mathbf r_{jk} = \mathbf x_j - \mathbf x_k\).
Givne $\mathbf g_j=\sum_{k\to j} w_{jk}\mathbf r_{jk}$, the eigenvector signs are fixed by axial references 
\[
    \mathbf h_j^{(q)}
    =
    \sum_{k\to j}
    w_{jk}\|\mathbf r_{jk}\|^q\mathbf r_{jk},
    \qquad
    \mathbf a_j=\mathbf h_j^{(1)}\times \mathbf g_j,
    \qquad
    \mathbf b_j=\mathbf h_j^{(2)}\times \mathbf g_j .
\]
This also gives a permutation-equivariant, \(\mathrm{SO}(3)\)-equivariant, and parity-invariant node frame whenever the local neighborhood is nondegenerate.


\subsection{Nonzero Spin-Weighted Gaunt Paths} 

The spin-weighted Gaunt coefficient is
\begin{equation}
\label{eq:swsh_gaunt_factorized_appendix}
G^{(\ell_3,m_3,s_3)}_{(\ell_1,m_1,s_1)(\ell_2,m_2,s_2)}
=
\sqrt{\frac{\prod_{i=1}^3(2\ell_i+1)}{4\pi}}
\begin{pmatrix}
\ell_1 & \ell_2 & \ell_3\\
m_1 & m_2 & -m_3
\end{pmatrix}
\begin{pmatrix}
\ell_1 & \ell_2 & \ell_3\\
s_1 & s_2 & -s_3
\end{pmatrix}.
\end{equation}
The scalar GTP case fixes \(s_1=s_2=s_3=0\), so the second Wigner symbol vanishes whenever \(\ell_1+\ell_2+\ell_3\) is odd. Allowing signed spin weights removes this limitation at the path level.

\begin{proposition}{Spin-weighted path completion}{}
Let $(\ell_1,\ell_2,\ell_3)$ satisfy the Clebsch-Gordan triangle rule. Then there exist signed spin weights
\[
s_1\in[-\ell_1,\ell_1],\qquad
s_2\in[-\ell_2,\ell_2],\qquad
s_3\in[-\ell_3,\ell_3],
\]
with $s_3=s_1+s_2$, such that the signed-spin path $(\ell_1,s_1)\otimes(\ell_2,s_2)\to(\ell_3,s_3)$
is not identically zero.
\end{proposition}

\begin{proof}
Choose
\[
    s_1=\ell_1,\qquad
    s_2=\ell_3-\ell_1,\qquad
    s_3=\ell_3 .
\]
Then \(s_3=s_1+s_2\) and $-\ell_2\le \ell_3-\ell_1\le \ell_2$, so all chosen spin weights are admissible. Consider the spin Wigner factor
\[
S
=
    \begin{pmatrix}
    \ell_1 & \ell_2 & \ell_3\\
    \ell_1 & \ell_3-\ell_1 & -\ell_3
    \end{pmatrix}.
\]
For this boundary case, the Wigner \(3j\) symbol has the explicit magnitude
\[
|S|
=
\left[
\frac{(2\ell_1)!(2\ell_3)!}
{(\ell_1-\ell_2+\ell_3)!(\ell_1+\ell_2+\ell_3+1)!}
\right]^{1/2}.
\]
All factorial arguments are nonnegative by the triangle inequalities, so $S\neq 0$. For these fixed spin weights, the spin-weighted Gaunt coefficient factors as
\[
G^{(\ell_3,m_3,s_3)}_{(\ell_1,m_1,s_1)(\ell_2,m_2,s_2)}
=
\sqrt{\frac{\prod_{i=1}^3(2\ell_i+1)}{4\pi}}  
\begin{pmatrix}
\ell_1 & \ell_2 & \ell_3\\
m_1 & m_2 & -m_3
\end{pmatrix} S.
\]
Taking
\[
    m_1=\ell_1,\qquad
    m_2=\ell_3-\ell_1,\qquad
    m_3=\ell_3
\]
gives the same nonzero Wigner factor in the magnetic part. Hence
\[
G^{(\ell_3,m_3,s_3)}_{(\ell_1,m_1,s_1)(\ell_2,m_2,s_2)}
=
\sqrt{\frac{\prod_{i=1}^3(2\ell_i+1)}{4\pi}} S^2
\neq 0 .
\]
Therefore the signed-spin Gaunt tensor has a nonzero entry and the
corresponding bilinear map is not identically zero.
\end{proof}

\section{Best Asymptotic Runtime Cost}
\label[appsec]{sec:time_cost}

\begin{proposition}{Runtime complexity for $|s|_{\max}=1$}{}
Given bounded multiplicities and a maximum angular degree $L$, a fully connected SpinGTP can be evaluated in $O(L^4\log^2L)$ using fast spherical transforms.
\end{proposition}

\begin{proof}
Let degrees $0\le \ell_1,\ell_2\le L$ and spin weights $-1\le s_1,s_2\le 1$. Define the spherical signals
\begin{equation}
\begin{aligned}
f(\Omega)
= \sum_{m_1=-\ell_1}^{\ell_1}
x_{\ell_1m_1s_1}^{(1)}
\,{}_{s_1}Y_{\ell_1m_1}(\Omega) \quad \text{and} \quad
g(\Omega)
= \sum_{m_2=-\ell_2}^{\ell_2}
x_{\ell_2m_2s_2}^{(2)}
\,{}_{s_2}Y_{\ell_2m_2}(\Omega).
\end{aligned}
\end{equation}
For $s_3=s_1+s_2$, the coefficient of the pointwise product of $f(\Omega)$ and $g(\Omega)$ is
\begin{equation}
\begin{aligned}
\int_{\mathrm S^2}
f(\Omega)g(\Omega)\,
{}_{s_3}Y_{\ell_3m_3}^{*}(\Omega)\,d\Omega
&=
\sum_{m_1,m_2}
G^{(\ell_3,m_3,s_3)}
 _{(\ell_1,m_1,s_1)(\ell_2,m_2,s_2)}
x_{\ell_1m_1s_1}^{(1)}
x_{\ell_2m_2s_2}^{(2)}.
\end{aligned}
\end{equation}
Thus, a tensor product path of $(\ell_1,s_1,\ell_2,s_2)$ is evaluated by two inverse spherical transforms,
one pointwise product, and one forward spherical transform. Since
$\ell_3\leq\ell_1+\ell_2\leq2L$, the forward and backward transform costs $O(L^2\log^2L)$ using fast spherical transform algorithm~\cite{healy2003ffts}. The pointwise multiplication itself costs $O(L^2)$. Therefore, a path of $(\ell_1,s_1,\ell_2,s_2)$ costs $O(L^2 \log^2 L) + O(L^2) = O(L^2 \log^2 L)$.

There are $(L+1)^2=O(L^2)$ input degree pairs of $(\ell_1,\ell_2)$. These pairs are evaluated separately for the independent path weights in~\Cref{eq:spin_gtp}. Each spherical transform returns all admissible output degrees $\ell_3$ simultaneously, so no additional factor of $L$ is required. Consequently, the time complexity of SpinGTP is
\begin{equation}
\begin{aligned}
T_{\mathrm{SpinGTP}}(L)
&=
O(L^2)\times O(L^2\log^2L)\\
&=
O(L^4\log^2L).
\end{aligned}
\end{equation}
This is the same asymptotic cost as the complete VSTP of~\cite{xie2026asymptotically}.
\end{proof}

\section{Real Basis Spin-Weighted Spherical Harmonics}
\label[appsec]{sec:real_swsh}


In computational physics and deep learning, while complex-valued representations are mathematically natural, real-valued features are often preferred for memory efficiency and compatibility with standard nonlinearities. We therefore store SWSH features in a real orthonormal basis and use a unitary change of basis only when evaluating complex Gaunt contractions.

\subsection{The Unitary $\mathbf{Q}$ Matrix}

Let \(\mathbf a_{\text{real}}\) denote real-basis coefficients and
\(\mathbf a_{\text{complex}}\) denote complex SWSH coefficients. We use a
unitary matrix \(\mathbf Q\) such that
\begin{equation}
    \mathbf a_{\text{complex}}=\mathbf Q\mathbf a_{\text{real}},
    \qquad
    \mathbf a_{\text{real}}
    =
    \operatorname{Re}\!\left(\mathbf Q^\dagger \mathbf a_{\text{complex}}\right).
\end{equation}
For the scalar case \(|s|=0\), \(\mathbf Q\in
\mathbb C^{(2\ell+1)\times(2\ell+1)}\). With rows and columns indexed by
\(m,k\in\{-\ell,\ldots,\ell\}\), the matrix elements used in our
implementation are
\begin{equation}
\mathbf Q_{m,k} = \begin{cases} 
1 & \text{if } m=0,\ k=0, \\[2pt]
\frac{1}{\sqrt{2}} & \text{if } m>0,\ k=m, \\[2pt]
\frac{i}{\sqrt{2}} & \text{if } m>0,\ k=-m, \\[2pt]
\frac{(-1)^{|m|}}{\sqrt{2}} & \text{if } m<0,\ k=|m|, \\[2pt]
\frac{-i(-1)^{|m|}}{\sqrt{2}} & \text{if } m<0,\ k=-|m|, \\[2pt]
0 & \text{otherwise}.
\end{cases}
\end{equation}
For \(|s|>0\), both signed spin sectors \(s\) and \(-s\) are retained.
The corresponding matrix has size
\[
    \mathbf Q\in
    \mathbb C^{2(2\ell+1)\times 2(2\ell+1)},
\]
with the complex coefficients ordered by the \(-|s|\) sector followed by the
\(|s|\) sector. This doubled matrix couples the paired \(\pm s\) and
\(\pm m\) components using~\Cref{eq:real_swsh}.

\subsection{Example: The $\mathbf{Q}$ Matrix for $\ell=1$}
For \(\ell=1\) and \(|s|=0\), with both rows and columns ordered as
\((-1,0,+1)\), the transformation matrix is
\begin{equation}
\mathbf{Q}^{(1)} =
\frac{1}{\sqrt{2}}
\begin{pmatrix}
 i & 0 & -1 \\
 0 & \sqrt{2} & 0 \\
 i & 0 & 1
\end{pmatrix}.
\end{equation}

By using this, real coefficients are mapped to the complex signed-spin basis, the generalized Gaunt contraction is evaluated, and the result is mapped back to the real basis using \(\mathbf Q^\dagger\).

\section{Derivation of the Parity-Equivariant Basis}
\label[appsec]{sec:parity_derivation}

A signal $f(\mathbf{x})$ is said to possess a definite parity $p \in \{+1, -1\}$ if, under the spatial inversion operator $P: \mathbf{x} \to -\mathbf{x}$, it satisfies the equivariance relation:
\begin{equation}
\label{eq:parity_def}
f(P\mathbf{x}) = p f(\mathbf{x}).
\end{equation}
For standard scalar spherical harmonics ($s=0$), parity is intrinsically tied to the degree $\ell$, where $Y_{\ell m}(-\mathbf{x}) = (-1)^\ell Y_{\ell m}(\mathbf{x})$. However, spin-weighted spherical harmonics exhibit a more complex behavior under inversion. As shown in~\Cref{sec:swsh_formula}, the parity transformation for SWSH involves a simultaneous inversion of the spin weight:
\begin{equation}
\label{eq:complex_parity}
{}_sY_{\ell m}(-\mathbf{x}) = (-1)^{\ell} {}_{-s}Y_{\ell m}(\mathbf{x}).
\end{equation}
Because the inversion maps a function of spin weight $s$ to a function of spin weight $-s$, the individual basis functions do not possess a well-defined parity unless $s=0$. This same coupling persists in the real-basis representation ${}_sR_{\ell m}$ defined in Eq.~\ref{eq:real_swsh}:
\begin{equation}
\label{eq:real_parity_relation}
{}_sR_{\ell m}(-\mathbf{x}) = (-1)^{\ell} {}_{-s}R_{\ell m}(\mathbf{x}).
\end{equation}
To construct a basis that satisfies the parity-equivariance constraint in Eq.~\ref{eq:parity_def}, we seek a linear combination of the spin-doubled basis functions. We define the Parity-Equivariant SWSH Basis, denoted $\mathcal{R}_{p}^{s,\ell,m}(\mathbf{x})$, as:
\begin{equation}
\label{eq:parity_basis_def}
\mathcal{R}_{p}^{s,\ell,m}(\mathbf{x}) = \frac{1}{\sqrt{2}} \left( {}_{-s}R_{\ell m}(\mathbf{x}) + p(-1)^\ell \, {}_{s}R_{\ell m}(\mathbf{x}) \right).
\end{equation}

\begin{proposition}{Parity Equivariance}{}
$\mathcal{R}_{p}^{s,\ell,m}$ satisfies parity equivariance relation~\Cref{eq:parity_def}.
\end{proposition}

\begin{proof}
To verify that this construction yields a basis with parity $p$, we apply the inversion operator $P$ such that
\begin{align}
\mathcal{R}_{p}^{s,\ell,m}(-\mathbf{x}) &= \frac{1}{\sqrt{2}} \left( {}_{-s}R_{\ell m}(-\mathbf{x}) + p(-1)^\ell \, {}_{s}R_{\ell m}(-\mathbf{x}) \right) \nonumber \\
&= \frac{1}{\sqrt{2}} \left( (-1)^\ell \, {}_{s}R_{\ell m}(\mathbf{x}) + p(-1)^\ell \, (-1)^\ell \, {}_{-s}R_{\ell m}(\mathbf{x}) \right) \nonumber \\
&= \frac{1}{\sqrt{2}} \left( (-1)^\ell \, {}_{s}R_{\ell m}(\mathbf{x}) + p \, {}_{-s}R_{\ell m}(\mathbf{x}) \right) \nonumber \\
&= p \left( \frac{1}{\sqrt{2}} \left( {}_{-s}R_{\ell m}(\mathbf{x}) + \frac{(-1)^\ell}{p} \, {}_{s}R_{\ell m}(\mathbf{x}) \right) \right) \nonumber \\
&= p \left( \frac{1}{\sqrt{2}} \left( {}_{-s}R_{\ell m}(\mathbf{x}) + p(-1)^\ell \, {}_{s}R_{\ell m}(\mathbf{x}) \right) \right) \nonumber \\
&= p \mathcal{R}_{p}^{s,\ell,m}(\mathbf{x}).
\end{align}
\end{proof}

This derivation confirms that the basis $\mathcal{R}_{p}^{s,\ell,m}(\mathbf{x})$ adheres strictly to the prescribed parity $p$. By decoupling the parity-even ($p=+1$) and parity-odd ($p=-1$) sectors, this formulation allows the network to explicitly represent pseudoscalar and axial-vector interactions, which are essential for resolving the antisymmetry gap in chiral geometric modeling.

\section{SpinGTP Implementation and Time Comparison}
\label[appsec]{sec:spingtp_impl}

We follow a similar implementation style with respect to e3nn~\cite{e3nn_software}.

\subsection{Tensor Product Implementation}
\label[appsec]{sec:tp_impl}

\paragraph{Multiplicity connection modes.}
All SpinGTP variants use the same real SWSH Gaunt kernel. For one instruction involving irreps
$\rho_1=(\ell_1,|s_1|,p_1)$, $\rho_2=(\ell_2,|s_2|,p_2)$, and $\rho_3=(\ell_3,|s_3|,p_3)$ with $p_3=p_1p_2$, let
\begin{equation}
\mathbf{K}^{\rho_3}_{\rho_1,\rho_2}(m_1,m_2,m_3)
\end{equation}
denote the real-basis SpinGTP kernel obtained from the signed-spin Gaunt contraction $G$ and parity recombination described in \Cref{eq:spin_gtp}. Here $m_1,m_2,m_3$ index the real irrep components, and $u,v,w$ index multiplicity channels. The different tensor-product layers differ only in how the multiplicity indices are connected.

\textbf{Fully connected tensor product.}
The fully connected mode (mode \texttt{uvw}) enumerates all admissible irrep-block pairs $(\rho_1,\rho_2)$ and all requested output irreps $\rho_3\in \rho_1\otimes\rho_2$. For each instruction, it first forms
\begin{equation}
R_{uv,m_3} = \sum_{m_1,m_2} \mathbf{K}^{\rho_3}_{\rho_1,\rho_2}(m_1,m_2,m_3) \mathbf{x}^{(1)}_{u,m_1}\mathbf{x}^{(2)}_{v,m_2},
\end{equation}
and then applies a learned dense multiplicity mixing
\begin{equation}
[\mathbf{x}^{(1)} \otimes_{\mathbf{W}} \mathbf{x}^{(2)}]^{(w,m_3)} = \sum_{u=1}^{U}\sum_{v=1}^{V} \mathbf{W}_{uvw}R_{uv,m_3}.
\end{equation}
The parameter count per instruction is $UVW$.

\textbf{Depthwise tensor product.}
The depthwise mode (mode \texttt{uvu}) is fully connected over admissible irrep-block pairs, but preserves the first input multiplicity channel. For each instruction, the output multiplicity is tied to the first input multiplicity ($w=u$), and the layer computes
\begin{equation}
[\mathbf{x}^{(1)} \otimes_{\mathbf{W}} \mathbf{x}^{(2)}]^{(u,m_3)} = \sum_{v=1}^{V} \mathbf{W}_{uv} \sum_{m_1,m_2} \mathbf{K}^{\rho_3}_{\rho_1,\rho_2}(m_1,m_2,m_3) \mathbf{x}^{(1)}_{u,m_1}\mathbf{x}^{(2)}_{v,m_2}.
\end{equation}
In the molecular message-passing setting, the second input is usually an edge basis with $V=1$. Following Equation (8) and (9) in the main text, we utilize the pre-contracted kernel $\mathbf{K}_{\mathrm{pre}}$ to evaluate the channel-depthwise form
\begin{equation}
[\mathbf{x}^{(1)} \otimes_{\mathbf{W}} \mathbf{x}^{(2)}]^{(u,m_3)} = \mathbf{W}_u \sum_{m_1} \mathbf{K}_{\mathrm{pre}}^{\rho_3}(m_1,m_3) \mathbf{x}^{(1)}_{u,m_1}.
\end{equation}
Thus the parameter count per instruction is $UV$, or $U$ when $V=1$. The optimized \texttt{cached\_uvu} and \texttt{triton\_uvu} backends are specialized for this common $V=1$ edge-feature case.

\textbf{Elementwise tensor product.}
The elementwise mode (mode \texttt{uuu}) pairs corresponding input irrep blocks and requires matching multiplicities. For each paired block, it computes channelwise products
\begin{equation}
R_{u,m_3} = \sum_{m_1,m_2} \mathbf{K}^{\rho_3}_{\rho_1,\rho_2}(m_1,m_2,m_3) \mathbf{x}^{(1)}_{u,m_1}\mathbf{x}^{(2)}_{u,m_2},
\end{equation}
followed by an optional per-channel weight
\begin{equation}
[\mathbf{x}^{(1)} \otimes_{\mathbf{W}} \mathbf{x}^{(2)}]^{(u,m_3)} = \mathbf{W}_u R_{u,m_3}.
\end{equation}
It does not mix different multiplicity channels and only connects corresponding input block pairs. The parameter count per instruction is $U$.

\subsection{Runtime Comparison}
\label[appsec]{sec:runtime_comparison}

We benchmark SpinGTP at two levels, including the depthwise tensor-product kernel and the full Equiformer-style model. All timings use a single NVIDIA H200 GPU and report time after warmup.

\paragraph{Depthwise tensor-product backend.}
We compare three implementations of the same depthwise tensor product, including direct Gaunt contraction, precontracted execution, and Triton-fused precontracted execution. The benchmark uses edge multiplicity one, \(s_{\max}=1\), and a depthwise path-expanded layout. The batch dimension is the number of independent edge-level tensor-product samples. We evaluate batch sizes \(\{1,32,128,1024\}\), and selectively report over these batch sizes for each multiplicity. For \(L_{\max}=1,2\), both input irreps and output irreps are up to \(L_{\max}\). For \(L_{\max}=3\), the first input is capped at \(L_{\max}^{(1)}=1\), while the second input and output remain at \(L_{\max}=3\).

\begin{table}[h]
\centering
\caption{
Depthwise SpinGTP tensor-product runtime on H200. Times are median milliseconds over 20 iterations. \(B\) is the number of independent edge-level tensor-product samples per call. Speedups are relative to direct Gaunt contraction.
}
\label{tab:tp_backend_runtime_all_lmax}
\begin{small}
{
\begin{tabular}{ccccccccc}
\toprule
\makecell{\(L_{\max}^{(1)}\)} &
\makecell{\(L_{\max}^{(2)}\)} &
\makecell{\(L_{\max}^{(\mathrm{out})}\)} &
\makecell{Multiplicity} &
\makecell{\(B\)} &
\makecell{Direct Gaunt\\(ms)} &
\makecell{Precontracted\\(ms)} &
\makecell{Triton Precontracted\\(ms)} &
\makecell{Gaunt / Triton} \\
\midrule
1 & 1 & 1 & 16  & 128  & 5.460  & 0.537 & 0.120 & \(45.4\times\) \\
1 & 1 & 1 & 64  & 128  & 5.493  & 0.466 & 0.119 & \(46.2\times\) \\
1 & 1 & 1 & 128 & 128  & 4.562  & 0.443 & 0.115 & \(39.7\times\) \\
\midrule
2 & 2 & 2 & 16  & 32   & 21.570 & 0.997 & 0.259 & \(83.15\times\) \\
2 & 2 & 2 & 64  & 32   & 29.912 & 1.769 & 0.748 & \(39.99\times\) \\
2 & 2 & 2 & 128 & 32   & 33.528 & 1.837 & 0.559 & \(59.96\times\) \\
\midrule
2 & 2 & 2 & 16  & 128  & 16.144 & 0.999 & 0.258 & \(62.5\times\) \\
2 & 2 & 2 & 64  & 128  & 29.561 & 1.771 & 0.752 & \(39.3\times\) \\
2 & 2 & 2 & 128 & 128  & 33.306 & 1.838 & 0.572 & \(58.2\times\) \\
\midrule

1 & 3 & 3 & 16  & 128  & 12.714 & 0.678 & 0.228 & \(55.9\times\) \\
1 & 3 & 3 & 64  & 128  & 20.610 & 1.080 & 0.608 & \(33.9\times\) \\
1 & 3 & 3 & 128 & 128  & 22.546 & 1.344 & 0.681 & \(33.1\times\) \\
\midrule
2 & 2 & 2 & 64  & 1024 & 29.912 & 1.777 & 2.090 & \(14.3\times\) \\
1 & 3 & 3 & 64  & 1024 & 23.941 & 1.394 & 1.865 & \(12.8\times\) \\
1 & 3 & 3 & 128 & 1024 & 22.628 & 1.350 & 2.240 & \(10.1\times\) \\
\bottomrule
\end{tabular}
}
\end{small}
\end{table}

Precontraction gives a consistent order-of-magnitude improvement over direct Gaunt contraction. Triton-fused precontraction is fastest for most small- and medium-batch settings. At large \(B\) and high multiplicity, the standard precontracted implementation can be faster because its batched matrix multiplications better saturate tensor cores.

\paragraph{End-to-end Equiformer timing at \(s=0\).}
We also compare the standard Equiformer against an Equiformer model using the SpinGTP tensor-product implementation. All SWSH irreps are constrained to \(s=0\), so the comparison isolates implementation speed. Both models use the same 6-layer Equiformer-style architecture, \(5.0~\text{\AA}\) cutoff, 500 maximum neighbors, 80 atoms per graph, no PBC, and energy-only forward evaluation.

The GTP-based model is \(1.21\times\)-\(1.89\times\) faster in forward evaluation. At practical batch sizes of 32-64 graphs, it sustains 65-72k atoms/s, compared with 38-44k atoms/s for the standard Equiformer. Because \(s=0\) is enforced, this gain comes from the tensor-product implementation.

\begin{table}[h]
\centering
\caption{
Full-model forward runtime for Equiformer and Equiformer with GTP at
\(s=0\). Times are median milliseconds over 10 iterations.
}
\label{tab:eq_swsh_forward_runtime}
\begin{small}
{
\begin{tabular}{ccccccc}
\toprule
\makecell{Graphs} &
\makecell{Atoms} &
\makecell{Equiformer\\(ms)} &
\makecell{GTP, \(s=0\)\\(ms)} &
\makecell{Speedup} &
\makecell{Equiformer\\(atoms/s)} &
\makecell{GTP, \(s=0\)\\(atoms/s)} \\
\midrule
1  & 80    & 23.438  & 19.358 & \(1.21\times\) & 3,413  & 4,133  \\
4  & 320   & 27.144  & 19.283 & \(1.41\times\) & 11,789 & 16,595 \\
16 & 1,280 & 41.729  & 22.049 & \(1.89\times\) & 30,674 & 58,053 \\
32 & 2,560 & 67.695  & 39.385 & \(1.72\times\) & 37,817 & 64,999 \\
64 & 5,120 & 116.298 & 70.917 & \(1.64\times\) & 44,025 & 72,197 \\
\bottomrule
\end{tabular}
}
\end{small}
\end{table}

\section{Training Details}
\label[appsec]{sec:training_details}

\subsection{3BPA} \label[appsec]{sec:3BPA training}
\paragraph{Architecture.}
For 3BPA, our model is built on top of MACE~\citep{batatia2022mace} with the following modifications. First, we use SWSH with $s_{max}=1$ instead of regular spherical harmonics, so each directed edge carries both standard and spin-weighted geometric features up to \(L_{\max}=3\). Second, the equivariant tensor product in each interaction block is replaced by the fully-connected SWSH tensor product, recovering the antisymmetric coupling paths absent from scalar GTP. Third, all linear layers throughout the network including embedding projections, skip connections, and readout layers, are replaced with the specialized SWSH linear layers described in~\cref{sec:swsh_layers}. Additionally, the \(s=0\) product basis uses the same body-ordered symmetric contraction as the MACE baseline, while \(|s|>0\) channels bypass the contraction and are updated via a linear map. Crucially, we add a graph frame as described in~\cref{sec:equiv} that assigns all edges within a molecule a shared gauge, making the scatter-sum over spin-weighted messages equivariant. 

\paragraph{Training Procedure.}
We optimize a weighted energy-and-force loss using Adam (AMSGrad) at learning
rate \(1\times10^{-2}\) and weight decay \(5\times10^{-7}\), under a
ReduceLROnPlateau schedule (factor \(0.8\), patience \(50\)). An exponential
moving average of model weights with decay \(0.99\) is maintained throughout
training. We follow a two-stage scheme: in Stage~1, the model is trained for
up to 2000 epochs with energy and force loss weights of \(1\) and \(1000\)
respectively. In Stage~2, we fine-tune with an increased energy loss weight
of \(25\) to better resolve energy differences between conformers, running
for up to 250 epochs. Final results are reported from Stage~2 checkpoints
selected on validation loss. All experiments use float32 precision with
RMS-forces output scaling. All 3BPA experiments were run on a single NVIDIA RTX A6000 (48\,GB) GPU, with each run taking approximately 63 GPU-hours (56 hours for Stage~1 and 7 hours for Stage~2). Full hyperparameters are summarized in Table~\ref{tab:train_setup_3bpa}.

\begin{table}[t]
\centering
\caption{Training Configuration for 3BPA.}
\label{tab:train_setup_3bpa}
\begin{tabular}{lll}
\toprule
Item & Setting \\
\midrule
Random seeds & \multicolumn{2}{l}{3, 8, 9} \\
Cutoff Radius & \multicolumn{2}{l}{5.0~\AA} \\
Radial basis & \multicolumn{2}{l}{8 Bessel, 5 polynomial cutoff basis} \\
Angular degree & \multicolumn{2}{l}{\(L_{\max}=3\), \(s_{\max}=1\)} \\
Interactions & \multicolumn{2}{l}{2} \\
Correlation order & \multicolumn{2}{l}{3} \\
Hidden irreps & \multicolumn{2}{l}{\(256{\times}(0,0)e + 256{\times}(1,0)o + 256{\times}(2,0)e + 16{\times}(1,1)o\)} \\
Readout MLP irreps & \multicolumn{2}{l}{\(16{\times}(0,0)e\)} \\
Optimizer & \multicolumn{2}{l}{Adam (AMSGrad)} \\
Learning rate & \multicolumn{2}{l}{\(1\times10^{-2}\)} \\
Weight decay & \multicolumn{2}{l}{\(5\times10^{-7}\)} \\
LR schedule & \multicolumn{2}{l}{ReduceLROnPlateau (factor 0.8, patience 50)} \\
Batch size & \multicolumn{2}{l}{5} \\
EMA decay & \multicolumn{2}{l}{0.99} \\
Output scaling & \multicolumn{2}{l}{RMS-forces scaling} \\
Precision & \multicolumn{2}{l}{float32} \\
\midrule
& \textit{Stage 1} & \textit{Stage 2} \\
Energy loss weight & 1 & 25 \\
Force loss weight & 1000 & 1000 \\
Max epochs & 2000 & 250 \\
Patience & 256 & -- \\
\bottomrule
\end{tabular}
\end{table}

\subsection{SPICE-MACE-OFF Chiral Subset}
\label[appsec]{sec:spice training}

\subsubsection{Chirality Classification and Parity Generalization}

\paragraph{Architecture.}
For chirality tasks, we adapt the SpinGTP framework to a graph classification objective by replacing the atomic energy readout with a global MLP-based head. The backbone utilizes spin-weighted spherical harmonics (SWSH) with $L_{\max}=2$ and $s_{\max}=2$ to encode the geometric sensitivity required for enantiomer differentiation. The hidden representation is defined by a specific irrep set: $64\times(0,0)e + 64\times(1,0)o + 8\times(1,1)o + 64\times(2,0)e$. To ensure a rigorous benchmark, baseline architectures including NequIP \cite{nequip}, MACE \cite{batatia2022mace}, and MACE-Gaunt \cite{luo2024enabling} are integrated into the same pipeline using identical cutoff radii ($r_{\max}=5.0$~\AA) and radial basis functions (Table~\ref{tab:train_setup_classification}).

\paragraph{Training Procedure.}
We evaluate model performance through two distinct protocols. First, Standard Classification involves a 2-class task (R vs.\ S) on a random 85/15 split, optimized via AdamW with a learning rate of $1\times 10^{-3}$. Second, Parity Generalization evaluates the model's ability to learn the underlying symmetry operation by training exclusively on R-chiral structures with a regression target of $y=+1$. Generalization is measured by sign accuracy on unseen S-chiral mirror images ($y=-1$). These experiments use Adam with a \texttt{StepLR} schedule and early stopping based on $\mathrm{acc}\_S$ to capture the model's peak generalization capability. We train models on 144 GB NVIDIA463 H200 GPUs. Detailed hyperparameter settings for both tasks are provided in Tables~\ref{tab:train_setup_classification} and \ref{tab:train_setup_parity}.

\begin{table}[h!]
\centering
\caption{Training Configuration for Chirality Classification.}
\label{tab:train_setup_classification}
\begin{tabular}{ll}
\toprule
Item & Setting \\
\midrule
Task & 2-class graph classification (R\_only vs.\ S\_only) \\
Data split & Random 85/15 train/validation split \\
Backbones & MACE, MACE-Gaunt, NequIP, SpinGTP \\
Cutoff & $r_{\max}=5.0$~\AA \\
SWSH structure & $L_{\max}=2$, 2 interactions, 8 Bessel, 5 cutoff basis \\
SWSH hidden irreps & $64{\times}(0,0)e + 64{\times}(1,0)o + 8{\times}(1,1)o + 64{\times}(2,0)e$ \\
Classifier head & MLP, hidden size 64, 2 output logits \\
Optimizer & AdamW \\
Learning rate & $1\times 10^{-3}$ \\
Weight decay & $1\times 10^{-5}$ \\
Batch size & 16 \\
Max epochs & 200 \\
LR schedule & ReduceLROnPlateau (factor 0.5, patience 10) \\
\bottomrule
\end{tabular}
\end{table}

\begin{table}[h!]
\centering
\caption{Training Configuration for Parity Generalization.}
\label{tab:train_setup_parity}
\begin{tabular}{ll}
\toprule
Item & Setting \\
\midrule
Task & Regression to parity target: $y=+1$ (R), $y=-1$ (S) \\
Train/test protocol & Train on R-only, evaluate on R-only and unseen S-only \\
Primary metric & Sign accuracy on unseen S ($\mathrm{acc}\_S$) \\
Loss & Mean-squared error (MSE) \\
Batch size & 16 \\
Max epochs & 200 \\
Learning rate & $1\times 10^{-3}$ \\
Weight decay & $1\times 10^{-5}$ \\
Patience & 30 on $\mathrm{acc}\_S$ \\
\midrule
SWSH setup & $r_{\max}=5.0$, mul=16, layers=2, $L_{\max}=2, s_{\max}=2$ \\
Optimizer & Adam + StepLR (step=50, $\gamma=0.5$) \\
\bottomrule
\end{tabular}
\end{table}

\subsubsection{Energy and Force Prediction}
\paragraph{Architecture.}
We use the same SpinGTP architecture as our 3BPA experiments (\cref{sec:3BPA training}, Table~\ref{tab:train_setup_3bpa}): spin-weighted edge spherical harmonics with \(s\in\{0,1\}\), fully-connected SWSH tensor products in place of scalar Clebsch–Gordan contractions, SWSH linear layers throughout, body-ordered symmetric contraction for the \(s=0\) product basis with direct linear updates for \(|s|>0\) channels, and a shared graph frame that fixes a consistent gauge for spin-weighted messages (\cref{sec:equiv}). The only substantive differences for SPICE are the hidden/readout widths (Table~\ref{tab:train_setup_spice}) and the training-time tensor-product backend configuration used for throughput on large-scale dataset training.

\paragraph{Training Procedure.}
We optimize a weighted energy-and-force loss with weights \(40\) (energy) and \(1000\) (forces), using Adam (AMSGrad) at learning rate \(1\times10^{-2}\) and weight decay \(5\times10^{-10}\), under a ReduceLROnPlateau schedule (factor \(0.8\), patience \(50\)). An exponential moving average of the model weights with decay \(0.99\) is maintained throughout training. Training runs for up to \(100\) epochs with patience \(50\), using float32 precision and RMS-forces output scaling. We parallelize across \(8\) GPUs via PyTorch DDP with per-GPU batch size \(32\) (global batch \(256\)). we train models on eight 144 GB NVIDIA463 H200 GPUs. Final results are reported from the best validation-loss checkpoint. All hyperparameters are summarized in Table~\ref{tab:train_setup_spice}.

\begin{table}[h!]
\centering
\caption{Training configuration for Energy and Force Prediction on SPICE chiral subset.}
\label{tab:train_setup_spice}
\begin{tabular}{ll}
\toprule
Item & Setting \\
\midrule
Random seed & 9 \\
Cutoff radius & 5.0~\AA \\
Radial basis & 8 Bessel, 5 polynomial cutoff basis \\
Angular degree & \(L_{\max}=3\), \(s_{\max}=1\) \\
Interactions & 2 \\
Correlation (body order) & 3 \\
Hidden irreps & \(192{\times}(0,0)e + 192{\times}(1,0)o + 192{\times}(2,0)e + 12{\times}(1,1)o\) \\
Readout MLP irreps & \(12{\times}(0,0)e\) \\
Product basis (\(s=0\)) & body-ordered symmetric contraction (MACE) \\
Optimizer & Adam (AMSGrad) \\
Learning rate & \(1\times10^{-2}\) \\
Weight decay & \(5\times10^{-10}\) \\
LR schedule & ReduceLROnPlateau (factor 0.8, patience 50) \\
Energy loss weight & 40 \\
Force loss weight & 1000 \\
EMA decay & 0.99 \\
Output scaling & RMS-forces scaling \\
Batch size & 32 per GPU (global 256, 8-GPU DDP) \\
Max epochs & 190 \\
Patience & 50 \\
Precision & float32 \\
\bottomrule
\end{tabular}
\end{table}

\subsection{OC20 IS2RE}
\label[appsec]{sec:OC20 training}

\paragraph{Architecture.}
For OC20, we use an SWSH-Equiformer model built on the Equiformer attention backbone. The model uses on-the-fly periodic graphs with cutoff radius \(5.0~\text{\AA}\), at most 500 neighbors, and 128 radial basis functions. The standard equivariant tensor products, linear layers, and normalization layers are replaced by their SWSH counterparts. The main OC20 configuration keeps the persistent node and edge feature space bounded at \(L_{\max}=1\). Spin is introduced through a local spin head constructing bounded \(|s|=1\) SWSH templates $(1,1)o$ in the local node frame and applying a SpinGTP path
\[
    (1,1)o\otimes (1,1)o\to (1,0)e.
\]
This head is enabled only in the attention activation path.

\paragraph{Training procedure.} We train on OC20 \textsc{IS2RE-Direct} on a single NVIDIA H200 GPU without force regression or noisy-node auxiliary supervision. The model is optimized with AdamW and a cosine learning-rate schedule with warmup. Checkpoints are selected by validation energy error. The full configuration is summarized in~\Cref{tab:train_setup_oc20}.

\begin{table}[h!]
\centering
\caption{Training configuration for OC20 \textsc{IS2RE-Direct}.}
\label{tab:train_setup_oc20}
\begin{small}
\begin{tabular}{ll}
\toprule
Item & Setting \\
\midrule
Task & OC20 \textsc{IS2RE-Direct} relaxed-energy prediction \\
Number of layers & 6 \\
Attention heads & 8 \\
Cutoff radius & \(5.0~\text{\AA}\) \\
Max neighbors & 500 \\
Radial basis & 128 Gaussian radial basis functions \\
Radial MLP & [64, 64] \\
Node embedding irreps & 256x(0,0)e+128x(1,0)e \\
Edge SWSH irreps & 1x(0,0)e+1x(1,0)e and 1x(1,1)o  \\
Attention head irreps & 32x(0,0)e+16x(1,0)e \\
MLP hidden irreps & 768x(0,0)e+384x(1,0)e \\
Output feature irreps & 512x(0,0)e \\
Optimizer & AdamW \\
Initial learning rate & \(2\times 10^{-4}\) \\
Weight decay & \(1\times 10^{-3}\) \\
Learning-rate schedule & Cosine \texttt{LambdaLR} \\
Warmup & 2 epochs, warmup factor \(0.2\) \\
Minimum LR factor & \(10^{-2}\) \\
Batch size & 32 \\
Evaluation batch size & 32 \\
Number of workers & 16 \\
EMA decay & 0.999 \\
Max epochs & 20 \\
\bottomrule
\end{tabular}
\end{small}
\end{table}

\end{document}